\documentclass{article}

\usepackage[nonatbib,final]{nips_2019}

\usepackage[utf8]{inputenc} \usepackage{hyperref}       \usepackage{url}            \usepackage{booktabs}       \usepackage{amsfonts}       \usepackage{nicefrac}       \usepackage{microtype}      \usepackage{tikz}

\title{Sliced Gromov-Wasserstein}

\author{
  Titouan Vayer \\
  Univ. Bretagne-Sud, CNRS, IRISA \\
  F-56000 Vannes\\
  \texttt{titouan.vayer@irisa.fr} \\
  \And
  R\'emi Flamary \\
  Univ. C\^ote d'Azur, CNRS, OCA Lagrange \\
  F-06000 Nice \\
  \texttt{remi.flamary@unice.fr} \\
  \AND
  Romain Tavenard \\
  Univ. Rennes, CNRS, LETG \\
  F-35000 Rennes \\
  \texttt{romain.tavenard@univ-rennes2.fr} \\
  \And
  Laetitia Chapel \\
  Univ. Bretagne-Sud, CNRS, IRISA \\
  F-56000 Vannes \\
  \texttt{laetitia.chapel@irisa.fr} \\
  \And
  Nicolas Courty \\
  Univ. Bretagne-Sud, CNRS, IRISA \\
  F-56000 Vannes \\
  \texttt{nicolas.courty@irisa.fr} \\
}

\usepackage{amsthm}
\usepackage{amsmath}
\usepackage{amssymb}
\usepackage{microtype}
\usepackage{graphicx}
\usepackage[lofdepth,lotdepth]{subfig}
\usepackage{booktabs} \usepackage{hyperref}
\usepackage{breakcites}
\usepackage{wrapfig}
\usepackage{xcolor}
\usepackage{caption}

\usepackage[mathscr]{euscript}
\usepackage{esint}

\usepackage{algorithm,algpseudocode}
\usepackage[normalem]{ulem}

\newcommand{\couplingset}{\Pi}
\newcommand{\Pm}{\mathcal{P}}
\newcommand{\R}{\mathbb{R}}
\newcommand{\E}{\mathbb{E}}
\newcommand{\Sp}{\mathbf{S}}
\newcommand{\gw}{GW}
\newcommand{\sgw}{SGW}
\newcommand{\risgw}{RISGW}
\newcommand{\gm}{GM}
\newcommand{\D}{\Delta}
\newcommand{\Sn}{S_{n}}
\newcommand{\insided}{c}
\newcommand{\supp}{\text{supp}}
\newcommand{\Stief}{\mathbb{V}_{p}(\R^{q})}

\newcommand{\xbf}{\mathbf{x}}
\newcommand{\ebf}{\mathbf{e}}
\newcommand{\ybf}{\mathbf{y}}

\newcommand{\sbf}{\mathbf{s}}
\def\G{\pi}
\def\GG{\boldsymbol\G}
\newcommand{\integ}[1]{{\{1,\dots,#1\}}}
\newcommand{\kron}{\otimes}
\newtheorem{prop}{Proposition}[section]
\newcommand{\ie}{\textit{i.e.}}

\def\epsilonb{\boldsymbol\epsilon}

\def\thetab{\boldsymbol\theta}
\newtheorem{corr}{Corollary}[section]

\def\X{{\bf X}}

\def\R{{\mathbb{R}}}

\def\tr{{\text{tr}}}
\def\one{{\mathbf{1}}}

\newcommand{\Qbf}{\mathbf{Q}}

\newcommand{\dr}{\mathrm{d}}

\theoremstyle{definition}

\theoremstyle{theorem}
\newtheorem{theorem}{Theorem}[section]

\theoremstyle{lemma}
\newtheorem{lemma}[theorem]{Lemma}

\definecolor{darkgreen}{rgb}{0,0.5,0}

\DeclareMathOperator*{\argmin}{arg\,min}
\DeclareMathOperator*{\argmax}{arg\,max}
\usepackage{etoolbox}
\makeatletter
\usepackage{ragged2e}

\begin{document}
\begin{titlepage}
\topskip0pt
\vspace*{\fill} 
\begin{center}
{\LARGE \bf To the reader's attention} \\
\vspace{1cm}
\justifying
{\large This paper contains an error in the proof of Theorem \ref{qap_main}. Indeed, as shown in \cite{Beinert}, this result is false and counterexamples exist, even if, in practice, numerical simulations suggest that Theorem \ref{qap_main} is ‘‘often'' true. More discussions can be found in \cite{Dumont}. The other theoretical and numerical results of this paper remain valid. We have decided to leave the paper as it is because we think it can contribute to the scientific discussion, which is one of trial and error.}
\vspace{1cm}
\begin{flushright}
\hfill The authors \\
20/10/2022
\end{flushright}
\end{center}
\vspace*{\fill}

\bibliographystyle{unsrt}

\end{titlepage}

\newpage

\maketitle

\begin{abstract}

Recently used in various machine learning contexts, the Gromov-Wasserstein distance ($\gw$) allows for comparing distributions whose supports do not necessarily lie in the same metric space. 
However, this Optimal Transport (OT) distance requires solving a complex non convex quadratic program which is most of the time very costly both in time and memory. 
Contrary to $\gw$, the Wasserstein distance ($W$) enjoys several properties ({\em e.g.} duality) that permit large scale optimization. Among those, the solution of $W$ on the real line, that only requires sorting
discrete samples in 1D, allows defining the Sliced Wasserstein ($SW$) distance. This paper proposes a new divergence based on $\gw$ akin to $SW$. 
We first derive a closed form for $\gw$ when dealing with 1D distributions, based on a
 new result for the related quadratic assignment problem. 
We then define a novel OT discrepancy that can deal with large scale distributions via a slicing approach and we show how it relates to the $\gw$ distance while being $O(n\log(n))$ to compute. We illustrate the behavior of this 
so called Sliced Gromov-Wasserstein ($\sgw$) discrepancy in experiments where we demonstrate its ability to tackle similar problems as $\gw$ while being several order of magnitudes
faster to compute. 
\end{abstract}

\section{Introduction}

Optimal Transport (OT) aims at defining ways to compare probability distributions.
One typical example is the Wasserstein distance ($W$) that has been used for varied tasks ranging from computer graphics~\cite{bonneel2016wasserstein} to signal processing~\cite{Thorpe2017}.
It has proved to be very useful for a wide range of machine learning tasks including generative modelling (Wasserstein GANs~\cite{arjovsky17a}), domain adaptation~\cite{courty2017optimal} or supervised embeddings for classification purposes~\cite{huang2016}. However one limitation of this approach is that it implicitly assumes \emph{aligned} distributions, \textit{i.e.} that lie in the same metric space or at least between spaces where a meaningful distance \emph{across} domains can be computed. From another perspective, the Gromov-Wasserstein ($\gw$) distance benefits from more flexibility when it comes to the more challenging scenario where heterogeneous distributions are involved, \emph{i.e.} distributions whose supports do not necessarily lie on the same metric space. It only requires modelling the topological or relational aspects of the distributions \emph{within} each domain in order to compare them. As such, it has recently received a high interest in the machine learning community, solving learning tasks such as heterogenous domain adaptation~\cite{ijcai2018-412}, deep metric alignment~\cite{gwcnn}, graph classification~\cite{vay_struc} or generative modelling~\cite{bunne_gan}.

OT is known to be a computationally difficult problem: the Wasserstein distance involves a linear program that most of the time prevents its use to settings with more than a few tens of thousands of points. For medium to large scale problems, some methods relying  \textit{e.g.} on entropic regularization \cite{cuturi2013sinkhorn} or dual formulation \cite{genevay_stochastic} have been investigated in the past years. Among them, one builds upon the mono-dimensional case where computing the Wasserstein distance can be trivially solved in $O(n \log n)$ by sorting points in order and pairing them from left to right. While this 1D case has a limited interest \emph{per se}, it is one of the main ingredients of the \emph{sliced} Wasserstein distance ($SW$)~\cite{rabin2011wasserstein}: high-dimensional data are linearly projected into sets of mono-dimensional distributions, the sliced Wasserstein distance being the average of the Wasserstein distances between all projected measures. 
This framework provides an efficient algorithm that can handle millions of
points and has similar properties to the Wasserstein distance \cite{bonotte_phd}. As
such, it has attracted attention and has been successfully used in various tasks
such as barycenter computation~\cite{bonneel:hal-00881872}, classification~\cite{Kolouri_2016_CVPR}
or generative modeling \cite{kolouri2018sliced,cvpr_sliced_gan,sliced_wass_flow_liutkus_2019,Wu_2019_CVPR}.

Regarding $\gw$, the optimization problem is a non-convex quadratic program, with
a prohibitive computational cost for problems with more than a few thousands of
points: the number of terms grows quadratically with the number of samples
and one cannot rely on a dual formulation as for Wasserstein. However several
approaches have been proposed to tackle its computation. Initially approximated
by a linear lower bound \cite{Chowdhury_memoli}, $\gw$ was thereafter estimated
through an entropy regularized version that can be efficiently computed by
iterating Sinkhorn
projections~\cite{Solomon:2016:EMA:2897824.2925903,peyre:hal-01322992}. More
recently a conditional gradient scheme relying on linear program OT solvers
was proposed in \cite{vay_struc}. However, as discussed more in detail in
Sec.~\ref{sec:gw:comp}, all these methods are still too
costly for large scale scenarii. 
In this paper,  we propose a new formulation related to $\gw$ that lowers
its computational cost. To that extent, we derive a novel OT discrepancy
called Sliced Gromov-Wasserstein ($\sgw$). It is similar in spirit to the Sliced
Wasserstein distance as it relies on the exact computation of 1D $\gw$ distances of
distributions projected onto random directions. We notably provide the first 1D
closed form solution of the $\gw$ problem by proving a new result about the
Quadratic Assignment Problem (QAP) for matrices that are squared euclidean
distances of real numbers.
Computation of $\sgw$ for discrete distributions of $n$ points is
$O(L \, n\log(n))$, where $L$ is the number of sampled directions. This
complexity is the same as the Sliced-Wasserstein distance and is even lower than computing the value of $\gw$ which is $O(n^3)$ for a known coupling (once the optimization
problem solved) in the general case~\cite{peyre:hal-01322992}. Experimental validation shows that $\sgw$
retains various properties of $\gw$ while being much cheaper to compute,
allowing its use in difficult large scale settings such as large mesh matching
or generative adversarial networks.

\paragraph{Notations}

The simplex histogram with $n$ bins will be denoted as  $\Sigma_{n}=\{a \in (\mathbb{R}_{+})^{n},\sum_{i} a_{i}=1\}$. 
For two histograms $a \in \Sigma_{n}$ and $b \in \Sigma_{m}$  we note $\couplingset(a,b)$ the set of all couplings of $a$ and $b$, \emph{i.e.} the set $\couplingset(a,b)=\{ \pi \in \mathbb{R}_{+}^{n \times m} \| \sum_{i}\pi_{i,j}=b_j ; \sum_{j}\pi_{i,j}=a_i \}$.  $\Sn$ is the set of all permutations of $\{1,...,n\}$.

We note $\|.\|_{k,p}$ the $\ell_k$ norm on $\R^{p}$. For any norm $\|.\|$ we note $d_{\|.\|}$ the distance induced by this norm.

$\delta_{x}$ is the dirac measure in $x$ \textit{s.t.} a discrete measure $\mu \in \Pm(\R^{p})$ can be written $\mu= \sum_{i=1}^{n} a_{i} \delta_{x_{i}}$ with $x_{i}\in\R^{p}$. For a continuous map $f : \R^{p} \rightarrow \R^{q}$ we note $f\#$ its push-forward operator. $f \#$ moves the positions of all the points in the support of the measure to define a new measure $f \# \mu \in \Pm(\R^{q})$ \textit{s.t.} $f \# \mu \stackrel{\text { def. }}{=} \sum_{i} a_{i} \delta_{f(x_{i})}$. We note $\mathcal{O}(p)$ the subset of $\R^{p\times p}$ of all orthogonal matrices. Finally $\Stief$ is the Stiefel manifold, \textit{i.e.} the set of all orthonormal $p$-frames in $\R^{q}$ or equivalently $\Stief=\{\D \in \R^{q\times p} | \D^{T}\D=I_{p} \}$.

\section{Gromov-Wasserstein distance}
\label{sec:back}

OT provides a way of inferring correspondences between two distributions by leveraging their
intrinsic geometries. If one has measures $\mu$ and $\nu$ on two spaces $X$ and
$Y$, OT aims at finding a correspondence (or \emph{transport}) map $\pi \in
\mathcal{P}(X\times Y)$ such that the marginals of $\pi$ are respectively $\mu$ and
$\nu$. When a meaningful distance or cost $c : X\times Y \mapsto \R_{+}$
\emph{across} the two domains can be computed, classical OT relies on minimizing
the total transportation cost between the two distributions $\int_{X \times Y}
c(x,y)d\pi(x,y)$ \textit{w.r.t.} $\pi$. The minimum total cost is often called
the Wasserstein distance between $\mu$ and $\nu$~\cite{Villani}.

However, this approach fails when a meaningful cost \emph{across} the
distributions cannot be defined, which is the case when $\mu$ and $\nu$ live
for instance in
Euclidean spaces of different dimensions or more generally when $X$ and
$Y$ are \emph{unaligned}, \textit{i.e.} when their features are not in
correspondence. This is particularly the case for features learned with deep learning as they
can usually be arbitrarily rotated or permuted. In this context, the $W$ distance with the naive cost $c(x,y)=\|x-y\|$ fails at capturing the similarity between the distributions. Some works address this issue by realigning spaces $X$ and
$Y$ using a global transformation before using the classical $W$
distance~\cite{pmlr-v89-alvarez-melis19a}. 
From another perspective, the so-called $\gw$ distance~\cite{memoli_gw} has been investigated in the past few years and rather relies on comparing \emph{intra}-domain distances $c_{X}$ and $c_{Y}$. 
\paragraph{Definition and basic properties}
	
Let $\mu \in \Pm(\R^{p})$ and $\nu \in \Pm(\R^{q})$ with $p\leq q$ be discrete measures on Euclidean spaces with $\mu= \sum_{i=1}^{n} a_{i} \delta_{x_{i}}$ and $\nu= \sum_{i=1}^{m} b_{j} \delta_{y_{j}}$ of supports $X$ and $Y$, where  $a \in \Sigma_{n}$ and $b \in \Sigma_{m}$ are histograms. 

Let $c_{X} : \R^{p} \times \R^{p} \rightarrow \R_{+}$ (\textit{resp.} $c_{Y}: \R^{q} \times \R^{q} \rightarrow \R_{+}$) measures the similarity between the samples in $\mu$ (\textit{resp.} $\nu$). 
The Gromov-Wasserstein ($\gw$) distance is defined as:
\begin{equation}
\label{gw}
\gw_{2}^{2}(c_{X},c_{Y},\mu,\nu)= \underset{\pi \in \couplingset(a,b)}{\min} J(c_{X},c_{Y},\pi)  \\
\end{equation}
where
\begin{equation*}
J(c_{X},c_{Y},\pi) =\sum_{i,j,k,l} \big| c_{X}(x_{i},x_{k})-c_{Y}(y_{j},y_{l}) \big|^{2} \pi_{i,j}\pi_{k,l}.
\end{equation*}
When $p=q$ and $c_X=c_Y$ we will write $\gw_{2}^{2}(c_{X},\mu,\nu)$ instead of $\gw_{2}^{2}(c_{X},c_{Y},\mu,\nu)$. The resulting coupling $\pi$ is a fuzzy correspondance map between the points of the distributions which tends to associate pairs of points with similar distances within each pair: the more similar $\insided_{X}(x_{i},x_{k})$ is to $\insided_{Y}(y_{j},y_{l})$, the stronger the transport coefficients $\pi_{i,j}$ and $\pi_{k,l}$ are. 
The $\gw$ distance enjoys many desirable properties when $c_{X}$ and $c_{Y}$ are distances so that $(X,c_{X},\mu)$ and $(Y,c_{Y},\nu)$ are called \emph{measurable metric spaces} (mm-spaces)~\cite{memoli_gw}. 
In this case, $\gw$ is a metric \textit{w.r.t.} the measure preserving isometries. 
More precisely, it is symmetric, satisfies the triangle inequality when considering three mm-spaces, and vanishes \textit{iff} the mm-spaces are \emph{isomorphic}, \textit{i.e.} when there exists a surjective function $f:X \rightarrow Y$ such that $f\#\mu=\nu$ ($f$ preserves the measures) and $\forall x,x' \in X^{2}, c_{Y}(f(x),f(x'))=c_{X}(x,x')$ ($f$ is an isometry). 
With a slight abuse of notations we will say that $\mu$ and $\nu$ are \emph{isomorphic} when
this occurs.  The $GW$ distance has several interesting properties, especially
in terms of invariances. It is clear from its formulation in eq. \eqref{gw} that it is
invariant to translations, permutations or rotations on both distributions when Euclidean distances are used. This last property
allows finding correspondences between complex word embeddings 
between different languages \cite{alvarez2018gromov}. Interestingly enough, when spaces have the same dimension, it has been proven that computing $\gw$ is equivalent to realigning both spaces using some linear transformation and then computing the $W$ distance on the realigned measures (Lemma 4.3 in \cite{pmlr-v89-alvarez-melis19a}).

$\gw$ can also be used with other similarity functions for $c_{X}$ and $c_{Y}$ (\textit{e.g.} kernels~\cite{peyre:hal-01322992} or squared integrable functions~\cite{Sturm2012}). 
In this work, we focus on squared euclidean distances, \textit{i.e.}  $c_{X}(x,x')=\|x-x'\|_{2,p}^{2}$, $c_{Y}(y,y')=\|y-y'\|_{2,q}^{2}$. This particular case is tackled by the theory of \emph{gauged measure spaces} \cite{Sturm2012,Chowdhury_memoli} where authors generalize mm-spaces with weaker assumptions on $c_{X},c_{Y}$ than the distance assumptions. More importantly in our context, invariants are the same as for distances since $\gw$ still vanishes $\textit{iff}$ there exists a measure preserving isometry (\emph{cf.} supplementary material) and the symmetry and triangle inequality are also preserved (see \cite{Chowdhury_memoli}).

\paragraph{Computational aspects}
\label{sec:gw:comp} The optimization problem \eqref{gw} is a non-convex Quadratic Program (QP).
Those problems are notoriously hard to solve since one cannot rely on convexity
and only descent methods converging to local minima are available. The problem
can be tackled by solving iterative linearizations of the quadratic function with a
conditional gradient as done in \cite{vay_struc}. In this case, each iteration
requires the optimization of a classical OT problem, that is $O(n^3)$. Another
approach consists in solving an approximation of problem \eqref{gw} by adding an
entropic regularization as proposed in \cite{peyre:hal-01322992}. This leads to
an efficient projected gradient algorithm where each iteration requires solving a regularized OT with the Sinkhorn algorithm that has be shown to
be nearly $O(n^2)$ and implemented efficiently on GPU. Still note that even though iterations for regularized $GW$
are faster, the computation of the final cost is $O(n^3)$ \cite[Remark 1]{peyre:hal-01322992}.

\section{From 1D GW to Sliced Gromov-Wasserstein}
\label{sec:sgw}

In this section, we first provide and prove a solution for an 1D
Quadratic Assignement Problem (QAP) with a quasilinear time complexity of $O(n\log(n))$. This new special case of the QAP is shown to be equivalent to the \emph{hard assignment} version of $\gw$, called the Gromov-Monge ($GM$) problem, with squared Euclidean cost for distributions lying on the real line. We also show that, in this context, solving $\gm$ is equivalent to solving $\gw$. We derive a new discrepancy named Sliced Gromov-Wasserstein ($\sgw$) that relies on these findings for efficient computation.

\paragraph{Solving a Quadratic Assignement Problem in 1D}

In Koopmans-Beckmann form~\cite{koopmans} a QAP takes as input two $n \times n$ matrices $A=(a_{ij})$, $B=(b_{ij})$. 
The goal is to find a permutation $\sigma \in \Sn$, the set of all permutations of $\{1, \cdots, n\}$, which minimizes the objective function {\small$ \sum_{i,j=1}^{n} a_{i,j} b_{\sigma(i),\sigma(j)}$}. 
In full generality this problem is NP-hard. 
However when matrices $A$ and $B$ have simple known structures, solutions can still be found (\textit{e.g.} diagonal structure such as Toeplitz matrix or separability properties such as $a_{i,j}=\alpha_{i}\alpha_{j}$ \cite{cela_new_2016,Cela-Schmuck-Wimer-Woeginger:2011,cela-2013}). 
We refer the reader to \cite{cela_all,Loiola_survey_qap} for comprehensive surveys on the QAP. 
The following theorem is a new result about QAP and states that it can be solved when $A$ and $B$ are squared Euclidean distance matrices of sorted real numbers: 

\begin{theorem}[A new special case for the Quadratic Assignment Problem]
\label{qap_main}
For real numbers $x_{1} < \dots < x_{n}$ and $y_{1} < \dots < y_{n}$, 
\begin{equation}
\label{eq:qap_main}
\underset{\sigma \in \Sn}{\min} \sum_{i,j}  - (x_{i}-x_{j})^{2}(y_{\sigma(i)}-y_{\sigma(j)})^{2}
\end{equation}
is achieved either by the identity permutation $\sigma(i)=i$ ($Id$) or the anti-identity permutation $\sigma(i)=n+1-i$ ($anti-Id$). In other words:
\begin{equation}
\exists \sigma \in \{Id,anti-Id\}, \ \sigma \in \underset{\sigma \in \Sn}{\argmin} \sum_{i,j}  - (x_{i}-x_{j})^{2}(y_{\sigma(i)}-y_{\sigma(j)})^{2} 
\end{equation}
\end{theorem}

To the best of our knowledge, this result is new. 
It states that if one wants to find the best one-to-one correspondence of real numbers such that their pairwise distances are best conserved, it suffices to sort the points and check whether the identity has a better cost than the anti-identity. 
Proof of this theorem can be found in the supplementary material. 
We postulate that this result also holds for $a_{ij}=|x_{i}-x_{j}|^{k}$ and $b_{ij}=-|y_{i}-y_{j}|^{k}$ with any $k\geq 1$ but leave this study for future works.

\paragraph {Gromov-Wasserstein distance on the real line}

When $n=m$ and $a_{i}=b_{j}=\frac{1}{n}$, one can look for the \emph{hard assignment} version of the $\gw$ distance resulting in the Gromov-Monge problem~\cite{memoli_gromov_monge_2018} associated with the following $\gm$ distance:
\begin{equation}
\gm_{2}(\insided_{X},\insided_{Y},\mu,\nu) = \min_{\sigma \in \Sn} \frac{1}{n^{2}} \sum_{i,j} \big| \insided_{X}(x_{i},x_{j})-\insided_{Y}(y_{\sigma(i)},y_{\sigma(j)}) \big|^{2} \label{eq:qapgw}
\end{equation}
where  $\sigma \in \Sn$ is a one-to-one mapping $\{1, \cdots, n\}  \rightarrow
\{1, \cdots, n\}$. Interestingly when the permutation $\sigma$ is known, the computation of the
cost is $O(n^2)$ which is far better than $O(n^3)$ for the general $\gw$ case. It is easy to see that this problem is equivalent to minimizing $ \sum_{i,j=1}^{n} a_{i,j} b_{\sigma(i),\sigma(j)}$ with $a_{ij}=\insided_{X}(x_{i},x_{j})$ and $b_{ij}=-\insided_{Y}(y_{\sigma(i)},y_{\sigma(j)})$. Thus, when a squared Euclidean cost is used for distributions lying on the real line, we exactly recover the solution of the $\gm$ problem defined in eq.~\eqref{eq:qap_main}. As matter of consequence, Theorem~\ref{qap_main} provides an efficient way of solving the Gromov-Monge problem.

Moreover, this theorem also allows finding a closed form for the $\gw$ distance. 
Indeed, some recent advances in graph matching state that, under some conditions on $A$ and $B$, the assignment problem is equivalent to its \emph{soft-assignment} counterpart~\cite{NIPS2018_7323}. 
This way, using both Theorem~\ref{qap_main} and~\cite{NIPS2018_7323}, one can find a solvable case for the $\gw$ distance when $p,q=1$ as stated in the following theorem: 
\begin{theorem}[Closed form for $\gw$ and $\gm$ in 1D for $n=m$ and uniform weights]
\label{sovable_gw_main}

Let $\mu= \frac{1}{n} \sum_{i=1}^{n} \delta_{x_{i}} \in \Pm(\R)$ and $\nu= \frac{1}{n} \sum_{i=1}^{n} \delta_{y_{j}} \in \Pm(\R)$ with $\R$ equipped with the Euclidean distance $d(x,x')=|x-x'|$. 
Then $\gw_{2}(d^{2},\mu,\nu)=\gm_{2}(d^{2},\mu,\nu)$. 

Moreover, if $x_{1} < \dots < x_{n}$ and $y_{1} < \dots < y_{n}$ this result is achieved either by the identity or the anti-identity permutation.
\end{theorem}

\begin{proof}[Sketch of the proof]
Since $d^{2}$ is conditionally negative definite of order 1 (see \textit{e.g.} Prop 3 and 4 in~\cite{NIPS2000_1862}), one can use the theory developed in~\cite{NIPS2018_7323} to prove that the assignment problem of $\gm$ is equivalent to $\gw$. 
Note that this result is true also for $c_{X}(x,x')=\|x-x'\|_{2,p}^{2}$ , $c_{Y}(y,y')=\|y-y'\|_{2,q}^{2}$ for any $p$ and $q$. Using Theorem~\ref{qap_main} for the $\gm$ distance concludes the proof.
\end{proof}

A more detailed proof is provided as supplementary material. In the following, we
only consider the case where $\mu$ and $\nu$ are discrete measures with the same
number of atoms $n=m$, uniform weights and $p\leq q$. Note also that, while both possible
solutions for problem \eqref{eq:qapgw} can be computed in $O(n\log(n))$, finding
the best one requires the computation of the cost which seems, at first sight, to have a $O(n^2)$ complexity. However, under the hypotheses of Theorem \ref{sovable_gw_main}, the cost can be computed in $O(n)$. Indeed, in this case, one can develop the sum in eq \eqref{eq:qapgw} to compute it in $O(n)$ operations using binomial expansion (see details in the supplementary materials) so that the overall complexity of finding the best assignment and computing the cost is $O(n\log(n))$ which is the same complexity as the Sliced Wasserstein distance.

\paragraph{Sliced Gromov-Wasserstein discrepancy} 
 Theorem \ref{sovable_gw_main} can be put in perspective with the Wasserstein distance for 1D distributions which is achieved by the identity permutation when points are sorted~\cite{OTFNT}. As explained in the introduction, this result was used to approximate the Wasserstein distance between measures of $\R^{p}$ using the so called Sliced Wasserstein (SW) distance~\cite{bonneel:hal-00881872}.{} 
The main idea is to project the points of the measures on lines of $\R^{p}$
where computing a Wasserstein distance is easy since it only involves a simple
sort and to average these distances. It has been proven that $SW$ and $W$ are
equivalent in terms of metric on compact domains~\cite{bonotte_phd}. 
In the same philosophy we build upon Theorem \ref{sovable_gw_main} to define a "sliced" version of the $\gw$ distance. In the following, we consider $\mu \in \Pm(\R^{p}),\nu \in \Pm(\R^{q})$ be probability distributions (not necessarily discrete).

Let $\mathbf{S}^{q-1}=\left\{\theta \in \mathbb{R}^{q} :\|\theta\|_{2,q}=1\right\}$ be
the $q$-dimensional hypersphere and $\lambda_{q-1}$ the uniform measure on
$\mathbf{S}^{q-1}$ . For $\theta$ we note $P_{\theta}$ the projection on $\theta$, \textit{i.e.} $P_{\theta}(x)=\langle x,
\theta \rangle$. For a linear map $\D \in \mathbb{R}^{q\times p}$ (identified with slight abuses of notation by its corresponding matrix),
we define the Sliced Gromov-Wasserstein (SGW) as follows:

\begin{equation}{}
\label{sgw}
\sgw_{\D}(\mu,\nu)= \underset{\theta \sim \lambda_{q-1}}{\E}[\gw^{2}_{2}(d^{2},P_{\theta}\#\mu_{\D},P_{\theta}\#\nu)]=  \fint_{\mathbf{S}^{q-1}}\gw^{2}_{2}(d^{2},P_{\theta}\#\mu_{\D},P_{\theta}\#\nu) d\theta 
\end{equation}

 where $\mu_{\D}=\D\#\mu \in \mathcal{P}(\R^{q})$ and $\fint_{\mathbf{S}^{q-1}}=\frac{1}{\text{vol}(\mathbf{S}^{q-1})}
 \int_{\mathbf{S}^{q-1}}$ is the normalized integral and can be seen as the expectation for $\theta$
 following a uniform distribution of support $\mathbf{S}^{q-1}$. The function $\D$ acts as a mapping for
 a point in $\R^{p}$ of the measure $\mu$ onto $\R^{q}$. When $p=q$ and when we consider $\D$ as the identity map we simply
 write $\sgw(\mu,\nu)$ instead of $\sgw_{I_{p}}(\mu,\nu)$. When $p < q$, one straightforward choice is $\D=\D_{pad}$ the "uplifting" operator which pads each point of the measure with zeros: $\D_{pad}(x)=(x_{1},\dots,x_{p},\underbrace{0,\dots,0}_{q-p})$. The procedure is illustrated in Fig \ref{sgw_figure}.

 In general fixing $\D$ implies that some properties of  $\gw$, such as the rotational invariance, are lost. Consequently, we also propose a variant of SGW that does not depends on the choice
 of $\D$ called Rotation Invariant SGW ($\risgw$) and expressed for $p\geq q$ as the following:
 \begin{equation}
  \label{risgw}
  \risgw(\mu,\nu)= \underset{\D \in \Stief}{\min}\sgw_{\D}(\mu,\nu).
  \end{equation}
 We propose to minimize $\sgw_{\D}$ with respect to $\D$ in the Stiefel
 manifold \cite{absil2009optimization} which can be seen as finding an optimal projector of the measure $\mu$ \cite{subspace_robust_wass_patty_2019,Deshpande_2019_CVPR}. This formulation comes at the cost of an
additional optimization step but allows recovering one key property of GW.
 When $p=q$ this encompasses for
 \textit{e.g.} all rotations of the space, making $\risgw$ invariant by rotation
 (see Theorem~\ref{propertiessgw_main}).

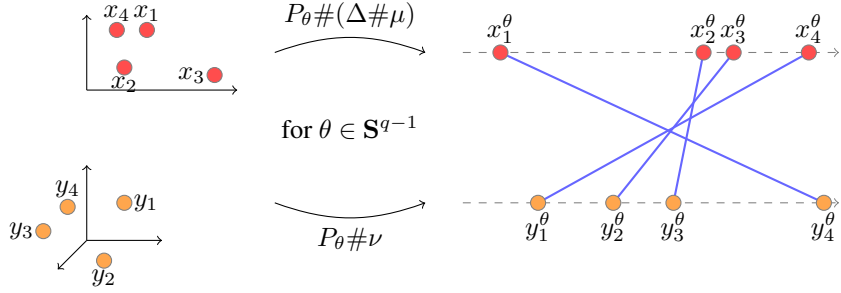
\begin{figure}
\begin{center}
\begin{tikzpicture}
\usetikzlibrary{decorations}
\usetikzlibrary{decorations.pathreplacing}

\draw[->] (0,0,0)--(1,0,0) ;
\draw[->] (0,0,0)--(0,1,0);
\draw[->] (0,0,0)--(0,0,1);

\fill[orange!70] (0.5,0.5,0) circle (0.1cm);
\fill[orange!70] (0.5,0,0.7) circle (0.1cm);
\fill[orange!70] (-0.5,0.2,0.2) circle (0.1cm);
\fill[orange!70] (-0.1,0.6,0.4) circle (0.1cm);
\draw[gray] (0.5,0.5,0) circle (0.1cm);
\draw[gray] (0.5,0,0.7) circle (0.1cm);
\draw[gray] (-0.5,0.2,0.2) circle (0.1cm);
\draw[gray] (-0.1,0.6,0.4) circle (0.1cm);
\draw (0.5,0.5,0) node[right]{$y_1$};
\draw (0.5,0,0.7) node[below]{$y_2$};
\draw (-0.5,0.2,0.2) node[left]{$y_3$};
\draw (-0.1,0.6,0.4) node[above]{$y_4$};

\draw[->] (0,2)--(2,2) ;
\draw[->] (0,2)--(0,3);
\fill[red!70] (0.4,2.8) circle (0.1cm);
\fill[red!70] (0.8,2.8) circle (0.1cm);
\fill[red!70] (0.5,2.3) circle (0.1cm);
\fill[red!70] (1.7,2.2) circle (0.1cm);
\draw[gray] (0.4,2.8) circle (0.1cm);
\draw[gray] (0.8,2.8) circle (0.1cm);
\draw[gray] (0.5,2.3) circle (0.1cm);
\draw[gray] (1.7,2.2) circle (0.1cm);
\draw (0.8,2.8) node[above]{$x_1$};
\draw (0.4,2.8) node[above]{$x_4$};
\draw (0.5,2.3) node[below]{$x_2$};
\draw (1.7,2.2) node[left]{$x_3$};

\draw[->] (2.5,0.5) to[out=340,in=200] (4.5,0.5);
\draw (3.5,0.3) node[below]{$P_{\theta}\#\nu$};

\draw[->] (2.5,2.5) to[out=20,in=160] (4.5,2.5);
\draw (3.5,2.7) node[above]{$P_{\theta}\#(\D\#\mu)$};

\draw (3.5,1.5) node{for $\theta$ $\in$ $\Sp^{q-1}$};
\foreach \x/\y in {2+4/5.6+4,3+4/4.6+4,3.8+4/4.2+4, 5.8+4/1.5+4} \draw[thick,blue!60](\y,2.5) -- (\x,0.5);

\draw[gray,dashed, ->]  (1+4,0.5) -- (6+4,0.5);
\foreach \y in {2+4,3+4,3.8+4, 5.8+4} \fill[orange!70] (\y,0.5) circle (0.1cm);
\foreach \y in {2+4,3+4,3.8+4, 5.8+4} \draw[gray] (\y,0.5) circle (0.1cm);
\foreach \y/\z in {2+4/1,3+4/2,3.8+4/3, 5.8+4/4} \draw(\y,0.5)node[below]{$y^{\theta}_\z$};

\draw[gray,dashed, ->]   (1+4,2.5) -- (6+4,2.5);
\foreach \x in {1.5+4,4.2+4,4.6+4, 5.6+4} \fill[red!70] (\x,2.5) circle (0.1cm);
\foreach \x in {1.5+4,4.2+4,4.6+4, 5.6+4} \draw[gray] (\x,2.5) circle (0.1cm);
\foreach \x/\z in {1.5+4/1,4.2+4/2,4.6+4/3, 5.6+4/4} \draw(\x,2.5)node[above]{$x^{\theta}_\z$};
\end{tikzpicture}
\end{center}
\caption{Example in dimension $p = 2$ and $q= 3$ (left) that are projected on the line (right). The solution for this projection is the anti-diagonal coupling. \label{sgw_figure}}
\end{figure}

Interestingly enough, $\sgw$ holds various properties of the $\gw$ distance as summarized in the following theorem:
\begin{theorem}{Properties of $\sgw$}
\label{propertiessgw_main}
\begin{itemize}
\item For all $\D$, $\sgw_{\D}$ and $\risgw$ are translation invariant. $\risgw$ is also rotational invariant when $p=q$, more precisely if $Q \in \mathcal{O}(p)$ is an orthogonal matrix, $\risgw(Q\#\mu,\nu)=\risgw(\mu,\nu)$ (same for any $Q' \in \mathcal{O}(p)$ applied on $\nu$).
\item $\sgw$ and $\risgw$ are pseudo-distances on $\Pm(\R^{p})$, \textit{i.e.} they are symmetric, satisfy the triangle inequality and $\sgw(\mu,\mu)=\risgw(\mu,\mu)=0$ .
\item Let $\mu,\nu \in \Pm(\R^{p})\times \Pm(\R^{p})$ be probability distributions with \emph{compact supports}. If $\sgw(\mu,\nu)=0$ then $\mu$ and $\nu$ are isomorphic for the distance induced by the $\ell_{1}$ norm on $\R^{p}$, \textit{i.e.} $d(x,x')=\sum_{i=1}^{p} |x_{i}-x_{i}'|$ for $(x,x') \in \R^{p} \times \R^{p}$. In particular this implies:
\begin{equation}
\sgw(\mu,\nu)=0 \implies \gw_{2}(d,\mu,\nu)=0
\end{equation} 
\end{itemize}
\end{theorem}

(with a slight abuse of notation we identify the matrix $Q$ by its linear application). A proof of this theorem can be found in the supplementary material. This theorem states that if $\sgw$ vanishes then measures must be isometric, as it is the case for $\gw$. It states also that $\risgw$ holds most of the properties of $\gw$ in term of invariances. 

\paragraph{Remark} The $\D$ map can also be used in the context of the Sliced Wasserstein distance so as to define $SW_{\D}(\mu,\nu)$, $RISW(\mu,\nu)$ for $\mu,\nu \in \Pm(\R^{p})\times \Pm(\R^{q})$ with $p\neq q$. Please note that from a purely computational point of view, complexities of these discrepancies are the same as $SGW$ and $RISGW$. Also, unlike $SGW$ and $RISGW$, these discrepancies are not translation invariant. More details are given in the supplementary material.

\paragraph{Computational aspects} In the following $\mu,\nu$ are \emph{discrete measures} with the \emph{same}
number of atoms $n=m$, and \emph{uniform weights}, \textit{i.e.} $\mu=\frac{1}{n}\sum_{i=1}^{n}\delta_{x_i},\nu=\frac{1}{n}\sum_{i=1}^{n}\delta_{y_i}$ with $x_i \in \R^{p},y_i \in \R^{q}$ so that we can apply Theorem \ref{sovable_gw_main}. Similarly to Sliced Wasserstein,
$\sgw$ can be approximated by replacing the integral by a finite sum over randomly
drawn directions. In practice we compute $\sgw$ as the average of $GW_2^2$
projected on $L$ directions $\theta$. While the sum in \eqref{sgw} can be implemented with libraries such as Pykeops \cite{charlier2018keops}, Theorem~\ref{sovable_gw} shows that computing~\eqref{sgw} is achieved by an
$O(n\log(n))$ sorting of the projected samples and by finding the optimal
permutation which is either the identity or the anti identity. Moreover computing the cost is $O(n)$ for each projection as explained previously. Thus the overall complexity of computing $\sgw$ with $L$ projections is
$O(Ln(p+q)+Ln\log(n)+Ln)=O(Ln(p+q+\log(n)))$ when taking into account the cost of projections. Note that these computations can be efficiently implemented in parallel on GPUs with modern toolkits such as Pytorch \cite{paszke2017automatic}.

The complexity of solving $\risgw$ is higher but one can rely on efficient algorithms for optimizing on the Stiefel manifold \cite{absil2009optimization} that have been implemented in several toolboxes \cite{townsend2016pymanopt,meghwanshi2018mctorch}. Note that each iteration in a manifold gradient decent requires the solution of $\sgw$, that can be computed and differentiated efficiently with the frameworks described above. Moreover, the optimization over the Stiefel manifold does not depend on the number of points but only on the dimension $d$ of the problem so that overall complexity is $n_{\text{iter}}(Ln(d+\log(n))+d^{3})$, which is affordable for small $d$. In practice, we observed in the numerical experiments that RISGW converges in few iterations (the order of $10$).

\section{Experimental results}
\label{sec:expe}

{The goal of this section is to validate $\sgw$ and its rotational invariant on both  quantitative (execution time) and qualitative sides. All the experiments were conducted on a standard  computer equipped with a NVIDIA Titan X GPU.}

\begin{figure}[t]
  \centering
  \includegraphics[width=.5\linewidth]{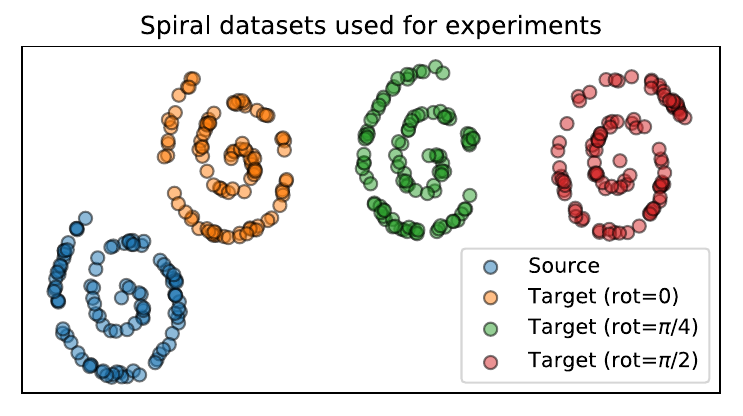}
  \includegraphics[width=.35\linewidth]{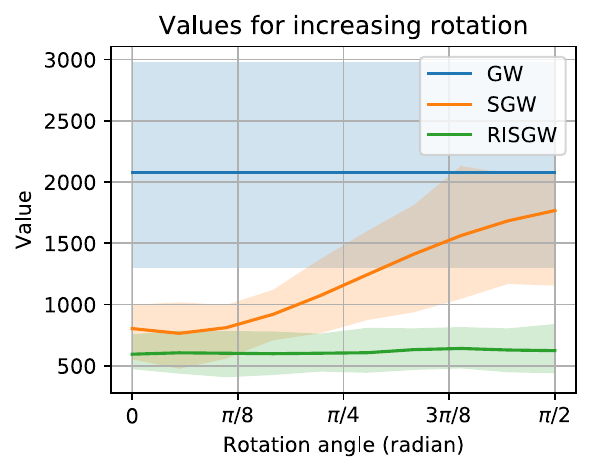}
  \caption{Illustration of $SGW$, $RISGW$ and $GW$ on spiral dataset for varying rotations on discrete 2D spiral dataset. (Left) Examples of spiral distributions for source and target with different rotations. (Right) Average value of $SGW$, $GW$ and $RISGW$ with $L=20$ as a function of rotation angle of the target. Colored areas correspond to the 20\% and 80\% percentiles. }
  \label{fig:spiral_example}
\end{figure}

\paragraph{SGW and RISGW on spiral dataset} As a first example, we use the spiral dataset from sklearn toolbox and compute $\gw$, $\sgw$ and $\risgw$ on $n=100$ samples with $L=20$  sampled lines for different rotations of the target distribution. The optimization of $\D$ on the Stiefel manifold is performed using Pymanopt \cite{townsend2016pymanopt} with automatic differentiation with autograd \cite{maclaurin2015autograd}.
Some examples of empirical distributions are available in Figure \ref{fig:spiral_example} (left). The mean value of $\gw$, $\sgw$ and $\risgw$ are reported on Figure \ref{fig:spiral_example} (right) where we can see that $\risgw$ is invariant to rotation as $\gw$ whereas $\sgw$ with $\D=I_d$ is clearly not.

\paragraph{Runtimes comparison \label{runtimes_comparaison}}

We perform a comparison between runtimes of $\sgw$, $\gw$ and its entropic
counterpart \cite{Solomon:2016:EMA:2897824.2925903}. We calculate these
distances between two 2D random measures of $n \in \{1e2,...,1e6\}$ points.
For $\sgw$, the number of projections $L$ is taken from $\{50,200\}$. We use the Python
Optimal Transport (POT) toolbox~\cite{flamary2017pot} to compute $\gw$ distance on CPU. For
entropic-$\gw$ we use the Pytorch GPU implementation from \cite{bunne_gan} that uses the log-stabilized Sinkhorn algorithm
\cite{Schmitzer_stab_sinkhorn} with a 
regularization parameter $\varepsilon=100$.
For $\sgw$, we implemented both a Numpy implementation and a Pytorch
implementation running on GPU.
Fig.
\ref{runtimes} illustrates the results. 

\begin{wrapfigure}{r}{.6\textwidth}
  \centering
  \includegraphics[width=\linewidth]{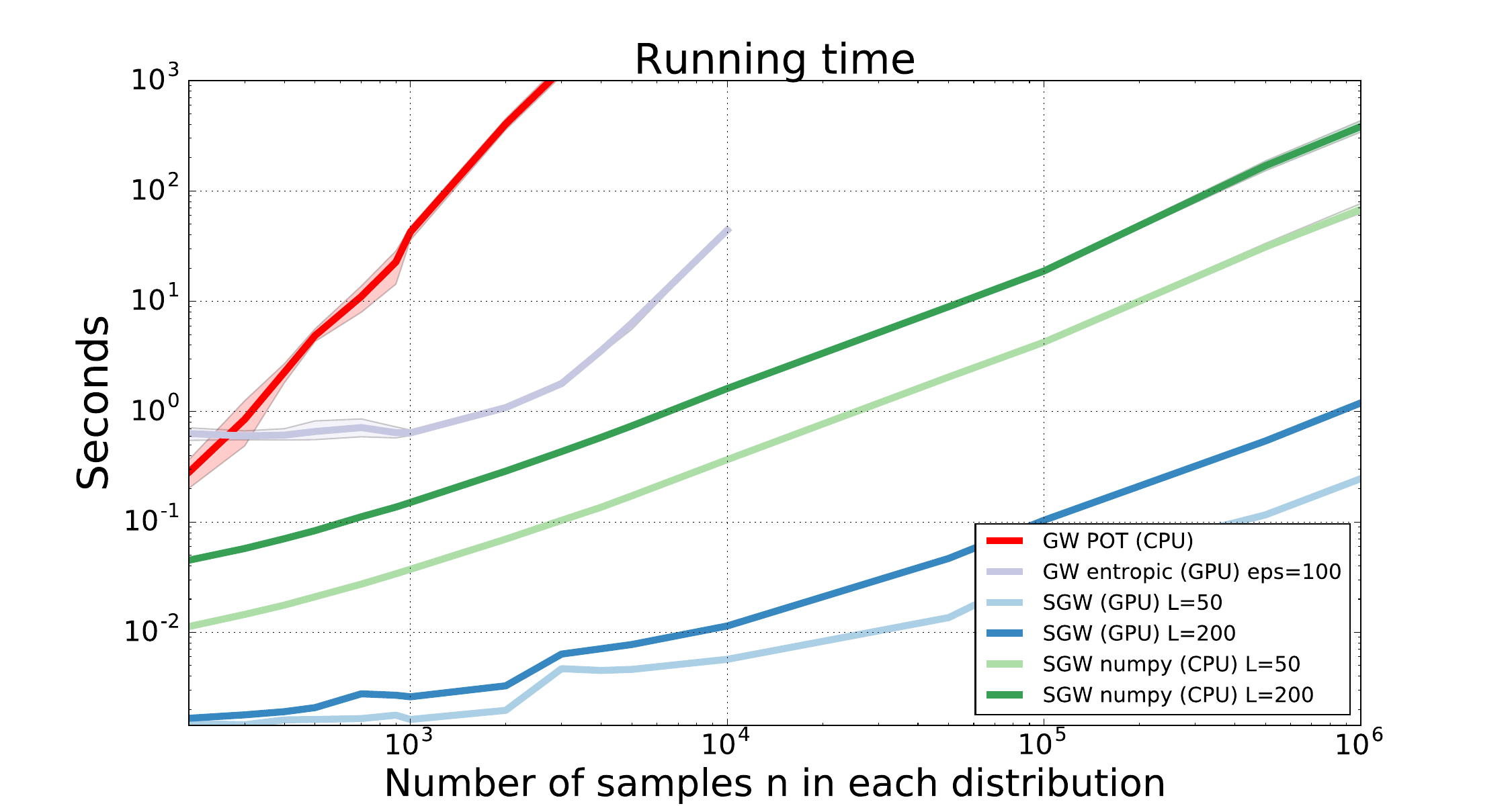}
  \caption{Runtimes comparison between $\sgw$, $\gw$, entropic-$\gw$ between two 2D random distributions with varying number of points from $0$ to $10^6$ in log-log scale. The time includes the calculation of the pair-to-pair distances. \label{runtimes}} 
  \end{wrapfigure}

$\sgw$ is the only method which scales
\textit{w.r.t.} the number of samples and allows computation for $n>10^4$. 
While entropic-$\gw$
uses GPU, it is still slow because the
gradient step size in the algorithm is inversely proportional to the regularization
parameter \cite{peyre:hal-01322992} which highly curtails the convergence of the
method.
On CPU, $\sgw$ is two orders of magnitude faster than $\gw$. On GPU, $\sgw$ is five orders of magnitude better than $\gw$
and four orders of magnitude better than entropic $\gw$. Still the slope of both
$\gw$ implementations are surprisingly good, probably due to their maximum
iteration stopping criteria. In this experiment we were able to compute $SGW$ between $10^6$ points in 1s. Finally note that we recover exactly a quasi-linear slope,
corresponding to the $O(n\log(n))$ complexity for $\sgw$.

\paragraph{Meshes comparison}

\begin{figure}
\centering
\includegraphics[width=0.6\linewidth]{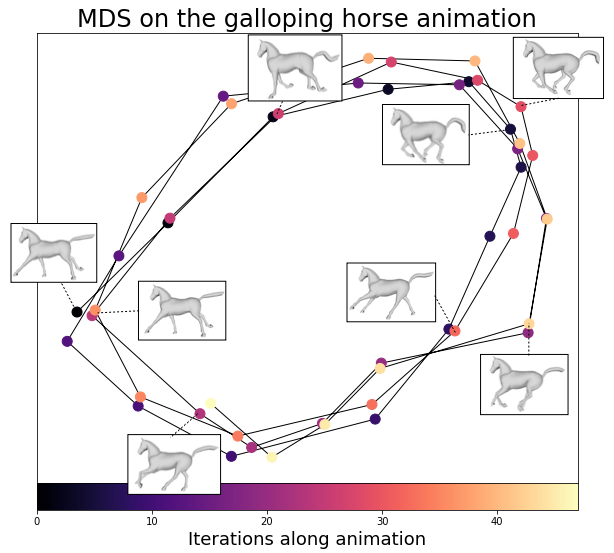}
\caption{Each sample in this Figure corresponds to a mesh and is colored by the corresponding time iteration. One can see that the cyclical nature of the motion is recovered.}\label{fig:gallop}
\end{figure}

In the context of computer graphics, $\gw$ can be used to quantify the correspondances between two meshes. A direct interest is found in shape retrieval, search, exploration or organization of databases. In order to recover experimentally some of the desired properties of the $\gw$ distance, we reproduce an experiment originally conducted in~\cite{rustamov2013map} and presented in~\cite{Solomon:2016:EMA:2897824.2925903} with the use of entropic-$\gw$.

From a given time series of 45 meshes representing a galloping horse, the goal is to conduct a multi-dimensional scaling (MDS) of the pairwise distances, computed with $\sgw$ between the meshes, that allows ploting each mesh as a 2D point. As one can observe in Fig.~\ref{fig:gallop}, the cyclical nature of this motion is recovered in this 2D plot, as already illustrated in~\cite{Solomon:2016:EMA:2897824.2925903} with the $\gw$ distance.  Each horse mesh is composed of approximately $9,000$ vertices. 
The average time for computing one distance is around 30 minutes using the POT implementation, which makes the computation of the full pairwise distance matrix impractical (as already mentioned in~\cite{Solomon:2016:EMA:2897824.2925903}). In contrast, our method only requires 25 minutes to compute the full distance matrix, with an average of 1.5s per mesh pair, using our CPU implementation. This clearly highlights the benefits of  our method in this case. 

\paragraph{SGW as a generative adversarial network (GAN)  loss}
In a recent paper~\cite{bunne_gan}, Bunne and colleagues propose a new variant of GAN between incomparable spaces, {\em i.e.} of different dimensions. In contrast with classical divergences such as Wasserstein, they suggest to capture the intrinsic relations between the samples of the target probability distribution by using $\gw$ as a loss for learning. More formally, this translates into the following optimization problem over a desired generator $G$:
\begin{equation}
G^* = \argmin \gw_{2}^{2}(c_{X},c_{G(Z)},\mu,\nu_G),
\end{equation}
where $Z$ is a random noise following a prescribed low-dimensional distribution (typically Gaussian), $G(Z)$ performs the uplifting of $Z$ in the desired dimensional space, and $c_{G(Z)}$ is the corresponding metric. $\mu$  and $\nu_G$ correspond respectively to the target and generated distributions, that we might want to align in the sense of $\gw$. Following the same idea, and the fact that sliced variants of the Wasserstein distance have been successfully used in the context of GAN~\cite{cvpr_sliced_gan}, we propose to use $\sgw$ instead of $\gw$ as a loss for learning $G$.
As a proof of concept, we reproduce the simple toy example of~\cite{bunne_gan}. Those examples consist in generating 2D or 3D distributions from target distributions either in 2D or 3D spaces (Fig.~\ref{fig:gan} and supplementary material). These distributions are formed by $3,000$ samples. We do not use their adversarial metric learning as it might confuse the objectives of this experiment and as it is not required for these low dimensional problems~\cite{bunne_gan}. The generator $G$ is designed as a simple multilayer perceptron with 2 hidden layers of respectively 256 and 128 units with ReLu activation functions, and one final layer with 2 or 3 output neurons (with linear activation) as output, depending on the experiment. The Adam optimizer is used, with a learning rate of $2.10^{-4}$ and $\beta_1=0.5,\beta_2=0.99$. The convergence to a visually acceptable solution takes a few hundred epochs. Contrary to~\cite{bunne_gan}, we directly back-propagate through our loss, without having to explicit a coupling matrix and resorting to the envelope Theorem.  Compared to~\cite{bunne_gan} and the use of entropic-$\gw$ , the time per epoch is more than one order of magnitude faster, as expected from previous experiment.

   \begin{figure}[!t]
   \centering
      \includegraphics[width=0.17\textwidth]{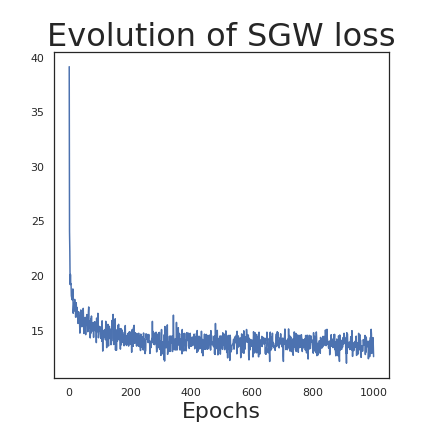}
      \includegraphics[width=0.16\textwidth]{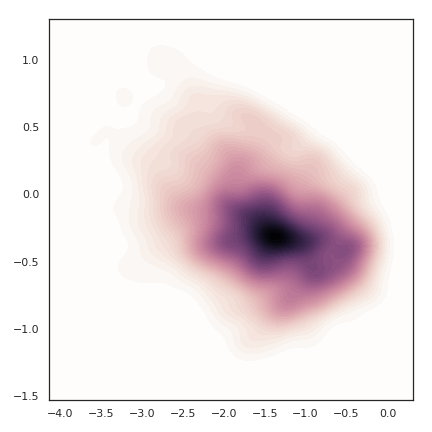}
      \includegraphics[width=0.16\textwidth]{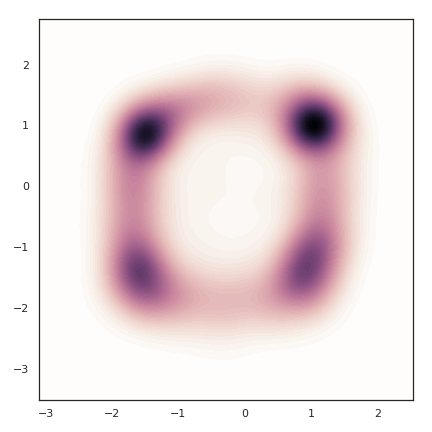}
      \includegraphics[width=0.16\textwidth]{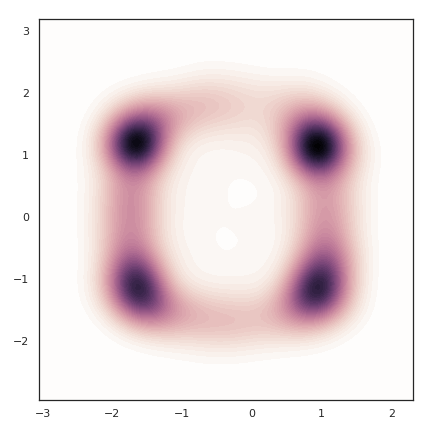}
      \includegraphics[width=0.16\textwidth]{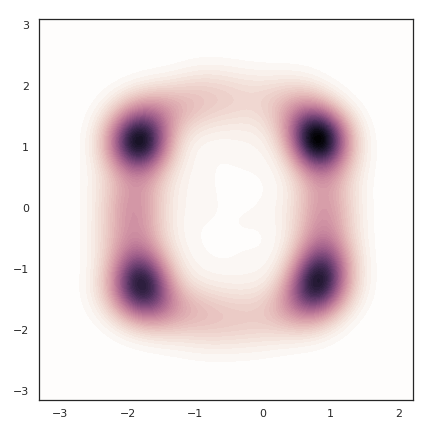}
      \includegraphics[width=0.16\textwidth]{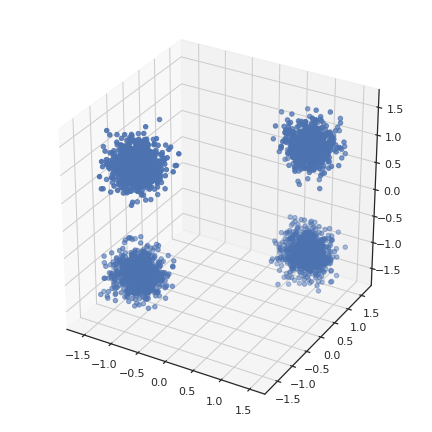}
     \caption{Using $\sgw$ in a GAN loss. First image shows the loss value along epochs. The next 4 images are produced by sampling the generated distribution ($3,000$ samples, plotted as a continuous density map). Last image shows the target 3D distribution.}
     \label{fig:gan}
   \end{figure}

\section{Discussion and conclusion}
In this work, we establish a new result about Quadratic Assignment Problem when
matrices are squared euclidean distances on the real line, and use it to state a closed form
expression for $\gw$ between monodimensional measures. Building upon this result
we define a new similarity measure, called the Sliced Gromov-Wasserstein and a
variant Rotation-invariant $\sgw$ and prove that both conserve various properties of the $\gw$
distance while being cheaper to compute and applicable in a large-scale setting.
Notably $\sgw$ can be computed in 1 second for distributions with 1 million samples each.  This paves the way for novel promising machine learning applications of optimal transport between metric spaces.

Yet, several questions are raised in this work. Notably, our method perfectly
fits the case when the two distributions are given empirically through samples
embedded in 
an Hilbertian space, that allows for projection on the real line.
This is the case in most of the machine learning applications that use the
Gromov-Wasserstein distance.
However, when only distances between samples are available, the projection
operation can not be carried anymore, while the computation of $\gw$ is still
possible. One can argue 
that it is possible to embed either isometrically those distances into a
Hilbertian space, or at least with a low distortion, and then apply the
presented technique. Our future line of work 
considers this option, as well as a possible direct reasoning on the distance
matrix. For example, one should be able to consider geodesic paths (in a graph
for instance) as the equivalent
appropriate geometric object related to the line. This constitutes the direct
follow-up of this work, as well as a better understanding of the accuracy of the
estimated discrepancy with respect to 
the ambiant dimension and the projections number. 

\subsection*{Acknowledgements}

We would like to thank Nicolas Klutchnikoff for the hepful discussions. This work benefited from the support from  OATMIL ANR-17-CE23-0012 project of the French National Research Agency (ANR). We gratefully acknowledge the support of NVIDIA Corporation with the donation of the Titan X GPU used for this research.

\bibliographystyle{unsrt}

\section{Suplementary materials}

\paragraph{Notations} In the following we will make the difference between a vector and a scalar by noting in bold vectors so that $\xbf \in \R^{n}$ and $x\in \R$. 
We note by $\mathcal{F}_{\mu}$ the Fourrier transform of $\mu$ for a probability measure $\mu \in \Pm(\R^{p})$. It is defined for $\sbf \in \R^{p}$ by $\mathcal{F}_{\mu}(\sbf)=\int e^{-2i \pi \langle \sbf,\xbf \rangle } d\mu(\xbf)$. We recall the rearrangement inequality:
\begin{theorem}[Rearrangement Inequality]
\label{memo:rearr}
Let $x_1 \leq \cdots \leq x_n$, $y_1 \leq \cdots \leq y_n$ then we have:
\begin{equation}
\forall \sigma \in \Sn, \ \sum_{i} x_{i}y_{n+1-i} \leq \sum_{i} x_{i}y_{\sigma(i)} \leq \sum_{i}x_{i}y_{i}
\end{equation}
If the numbers are different then the lower bound (resp upper bound) is attained only for the permutation which reverses the order (resp for the identiy permutation)
\end{theorem}

\section{Proof for the QAP}

In this section we aim at proving the new special case of the QAP, which is recalled in the next theorem:

\begin{theorem}[A new special case of the QAP.]
\label{qap}
For real numbers $x_{1} < \dots < x_{n}$ and $y_{1} < \dots < y_{n}$, 
\begin{equation}
\label{eq:qap}
\underset{\sigma \in \Sn}{\min} \sum_{i,j}  - (x_{i}-x_{j})^{2}(y_{\sigma(i)}-y_{\sigma(j)})^{2}
\end{equation}
is achieved either by the identity permutation $\sigma(i)=i$ ($Id$) or the anti-identity permutation $\sigma(i)=n+1-i$ ($anti-Id$). In other words:
\begin{equation}
\exists \sigma \in \{Id,anti-Id\}, \ \sigma \in \underset{\sigma \in \Sn}{\argmin} \sum_{i,j}  - (x_{i}-x_{j})^{2}(y_{\sigma(i)}-y_{\sigma(j)})^{2} 
\end{equation}
\end{theorem}

Let us note $\mathcal{I}=\{\xbf,\ybf \in \mathbb{R}^{n}\times \mathbb{R}^{n} | x_{1} < x_{2} < \dots < x_{n} \ , y_{1} < y_{2} < \dots < y_{n} \}$. We consider for $\xbf,\ybf \in \mathcal{I}$:
\begin{equation}
\label{qpproblem}
\underset{\sigma \in \Sn}{\text{max}} \ Z(\xbf,\ybf,\sigma)=\underset{\sigma \in \Sn}{\text{max}} \sum_{i,j} (x_{i}-x_{j})^{2} (y_{\sigma(i)}-y_{\sigma(j)})^{2}
\end{equation}
The original problem is equivalent to maximizing $Z(\xbf,\ybf,\sigma)$ over $\Sn$. Given $\xbf,\ybf \in \mathcal{I}$, we define $X \stackrel{\text{def}}{=}\sum_{i} x_{i}$ and $Y \stackrel{\text{def}}{=}\sum_{i} y_{i}$. Then:
\begin{equation*}
\begin{split}
&\underset{\sigma \in \Sn}{\text{max}} \ Z(\xbf,\ybf,\sigma) =\underset{\sigma \in \Sn}{\text{max}} \sum_{i,j} (x_{i}-x_{j})^{2} (y_{\sigma(i)}-y_{\sigma(j)})^{2} \\
&=\underset{\sigma \in \Sn}{\text{max}} \sum_{i,j} (x_{i}^{2}+x_{j}^{2})(y_{\sigma(i)}^{2}+y_{\sigma(j)}^{2}) -2 \sum_{i,j} x_{i} x_{j}(y_{\sigma(i)}^{2}+y_{\sigma(j)}^{2}) -2 \sum_{i,j} y_{\sigma(i)} y_{\sigma(j)}(x_{i}^{2}+x_{j}^{2}) \\ 
&+4 \sum_{i,j} x_{i} x_{j} y_{\sigma(i)} y_{\sigma(j)} \\
&= \underset{\sigma \in \Sn}{\text{max}} \ 2 n \sum_{i} x_{i}^{2}y_{\sigma(i)}^{2}-2 \sum_{i,j} x_{i} x_{j}(y_{\sigma(i)}^{2}+y_{\sigma(j)}^{2}) -2 \sum_{i,j} y_{\sigma(i)} y_{\sigma(j)}(x_{i}^{2}+x_{j}^{2}) \\ 
&+4 \sum_{i,j} x_{i} x_{j} y_{\sigma(i)} y_{\sigma(j)} + 2 (\sum_{i} x_{i}^{2})(\sum_{i} y_{i}^{2}) \\
&= \underset{\sigma \in \Sn}{\text{max}} \ 2 n \sum_{i} x_{i}^{2}y_{\sigma(i)}^{2}-4 X \sum_{i} x_{i} y_{\sigma(i)}^{2} -4 Y \sum_{i} x_{i}^{2} y_{\sigma(i)} +4 \sum_{i,j} x_{i} x_{j} y_{\sigma(i)} y_{\sigma(j)} + 2 (\sum_{i} x_{i}^{2})(\sum_{i} y_{i}^{2}) \\
&\stackrel{(*)}{=} Cte+2 \big(\underset{\sigma \in \Sn}{\text{max}} \  \sum_{i} n x_{i}^{2}y_{\sigma(i)}^{2}-2 \sum_{i} (X x_{i} y_{\sigma(i)}^{2} + Y x_{i}^{2} y_{\sigma(i)}) + 2(\sum_{i} x_{i} y_{\sigma(i)})^{2}\big) \\
\end{split}
\end{equation*}
where in (*) we defined $Cte\stackrel{\text{def}}{=}2 (\sum_{i} x_{i}^{2})(\sum_{i} y_{i}^{2})$ the term that does not depend on $\sigma$. Overall we have:
\begin{equation}
\label{equivone}
\forall \xbf,\ybf \in \mathcal{I}, \ \underset{\sigma \in \Sn}{\text{argmax}} \ Z(\xbf,\ybf,\sigma) = \underset{\sigma \in \Sn}{\text{argmax}} \sum_{i} n x_{i}^{2}y_{\sigma(i)}^{2}-2 \sum_{i} (X x_{i} y_{\sigma(i)}^{2} + Y x_{i}^{2} y_{\sigma(i)}) + 2(\sum_{i} x_{i} y_{\sigma(i)})^{2}
\end{equation}

Since $Z$ is invariant by translation of $\xbf,\ybf$ we can suppose without loss of generality that $X=Y=0$. We consider the set $\mathcal{D}=\{\xbf,\ybf \in \mathbb{R}^{n}\times \mathbb{R}^{n} | x_{1} < x_{2} < \dots < x_{n} \ , y_{1} < y_{2} < \dots < y_{n}, \ \sum_i x_i= \sum_j y_j=0 \}$. We want to find for $\xbf,\ybf \in \mathcal D$:
\begin{equation}
\tag{QAP}
\label{eq:qap_to_prove}
\max_{\sigma \in \Sn} n \sum_{i} x_i^{2} y_{\sigma(i)}^{2} +2 \left(\sum_i x_i y_{\sigma(i)}\right)^{2} \stackrel{def}{=} \max_{\sigma \in \Sn} g(\xbf,\ybf,\sigma)
\end{equation}

We have the following result:

\begin{lemma}
\label{lemma:equivalence_both_probs}
Let $\xbf,\ybf \in \mathcal{D}$ and consider the problem: 
\begin{equation}
\tag{QP}
\label{eq:qp_equiv}
\max_{\GG \in DS} \sum_{ijkl} (x_i^2 y_j^2+ 2 x_i y_j x_k y_l)\pi_{ij} \pi_{kl}\\
\end{equation}
where $DS$ is the set of doubly stochastic matrices. Then \eqref{eq:qp_equiv} and \eqref{eq:qap_to_prove} are equivalent. More precisely if $\sigma^{*}$ is an optimal solution of \eqref{eq:qap_to_prove} then $\GG_{\sigma^{*}}$ defined by $\pi_{\sigma^{*}}(i,j)=1$ if $j=\sigma^{*}(i)$ else $0$ for all $(i,j) \in \integ{n}^{2}$ is an optimal solution of \eqref{eq:qp_equiv} and if $\GG^{*}$ is an optimal solution of \eqref{eq:qp_equiv} then it is supported on a permutation $\sigma^{*}$ which is an optimal solution of \eqref{eq:qap_to_prove}.
\end{lemma}

\begin{proof}
The problem \eqref{eq:qap_to_prove} can be rewritten as:
{\small
\begin{equation}
\label{eq:QAP_with_P}
\begin{split}
&\max_{\begin{smallmatrix}P_{ij}\in\{0,1\} \\ \forall j \sum_i P_{ij}=1 \\ \forall i \sum_j P_{ij}=1 \end{smallmatrix}} n \sum_{ij} x_i^{2} y_{j}^{2} P_{ij} +2 \left(\sum_{i,j} x_i y_{j}P_{ij}\right)^{2} =\max_{\begin{smallmatrix}P_{ij}\in\{0,1\} \\ \forall j \sum_i P_{ij}=1 \\ \forall i \sum_j P_{ij}=1 \end{smallmatrix}} n \sum_{ij} x_i^{2} y_{j}^{2} P_{ij} +2 \sum_{ijkl}x_i x_k y_j y_l P_{ij} P_{kl} \\
&\stackrel{*}{=}\max_{\begin{smallmatrix}P_{ij}\in\{0,1\} \\ \forall j \sum_i P_{ij}=1 \\ \forall i \sum_j P_{ij}=1 \end{smallmatrix}} \sum_{ijkl} x_i^2 y_j^2 P_{ij} P_{kl} + 2 \sum_{ijkl}x_i y_j x_k y_l P_{ij} P_{kl}= \max_{\begin{smallmatrix}P_{ij}\in\{0,1\} \\ \forall j \sum_i P_{ij}=1 \\ \forall i \sum_j P_{ij}=1 \end{smallmatrix}} \sum_{ijkl} (x_i^2 y_j^2+ 2 x_i y_j x_k y_l)P_{ij} P_{kl}\\
\end{split}
\end{equation}
}
In (*) we used $\sum_{k,l} P_{k,l}=n$. We consider the following relaxation of \eqref{eq:QAP_with_P} as:
\begin{equation}
\max_{\GG \in DS} \sum_{ijkl} (x_i^2 y_j^2+ 2 x_i y_j x_k y_l)\pi_{ij} \pi_{kl}\\
\end{equation}
which is a maximization of a convex function. More precislely it is quadratic programming problem which Hessian is positive semi-definite $\xbf \xbf^{T}\kron \ybf \ybf^{T}$. Since the problem is a maximization of a convex function an optimal solution $\GG^{*}$ of \eqref{eq:qp_equiv} lies necassarily in the extremal points of $DS$ \cite{rockafellar-1970a} such that both \eqref{eq:qp_equiv} and \eqref{eq:qap_to_prove} are equivalent: if $\GG^{*}$ is an optimal solution it is necessarily supported on a $\sigma^{*} \in \Sn$ such that $\sigma^{*}$ is an optimal solution of \eqref{eq:qap_to_prove} and if $\sigma^{*} \in \Sn$ is an optimal solution of \eqref{eq:qap_to_prove} then $\GG^{*}$ defined by $\pi^{*}_{ij}=1$ if $j=\sigma^{*}(i)$ else $0$ for all $(i,j) \in \integ{n}^{2}$ is an optimal solution of \eqref{eq:qp_equiv}.

\end{proof}

\begin{lemma}
\label{lemma:gg_two_cases}
Let $\xbf,\ybf \in \mathcal D$. For $\sigma \in \Sn$ we note $C(\xbf,\ybf,\sigma)=\sum_{i}x_i y_{\sigma(i)}$. Let $\GG^{*}$ an optimal solution of \eqref{eq:qp_equiv} with $\sigma^{*}$ the permutation associated to $\GG^{*}$.

If $C(\xbf,\ybf,\sigma^{*})>0$ then $\GG^{*}=\mathbf{I_n}$ is the identiy and if $C(\xbf,\ybf,\sigma^{*})<0$ then $\GG^{*}=\mathbf{J_n}$ is the anti-identity. 
\end{lemma}

To prove this result we will rely on the following theorem which gives necessary conditions for being an optimal solution of \eqref{eq:qp_equiv}:

\begin{theorem}[Theorem 1.12 in \cite{murty_linear_1988}]
\label{murty_theo}
Consider the following (QP):
\begin{equation}
\label{eq:qp_general}
\begin{array}{cl}{\min _{\xbf} f(\xbf)} & {=\mathbf{c} \xbf+\xbf^{T} \mathbf{Q} \xbf} \\ {\text {s.t.}} & {\mathbf{A} \xbf = \mathbf{b}},\;  {\xbf \geq 0}\end{array}
\end{equation}
Then if $\xbf_{*}$ is an optimal solution of \eqref{eq:qp_general} it is an optimal solution of the following (LP):
\begin{equation}
\label{eq:lp_general}
\begin{array}{cl}{\min _{\xbf} f(\xbf)} & {=(\mathbf{c} + \xbf_{*}^{T}\mathbf{Q} )  \xbf} \\ {\text {s.t.}} & {\mathbf{A} \xbf = \mathbf{b}},\;  {\xbf \geq 0}\end{array}
\end{equation}
\end{theorem}

\begin{proof}{Of Lemma \ref{lemma:gg_two_cases}.}
Applying Theorem \ref{murty_theo} in our case gives that if $\GG^{*}$ is a solution of \eqref{eq:qp_equiv} it necessarily a solution of the following (LP):
\begin{equation}
\max_{\GG \in DS} \sum_{ijkl} (x_i^2 y_j^2+ 2 x_i y_j x_k y_l)\pi^{*}_{ij} \pi_{kl}=n \sum_{ij} x_i^{2} y_{j}^{2} \pi^{*}_{ij}+  \max_{\GG \in DS} 2(\sum_{ij} x_i y_j \pi^{*}_{ij}) (\sum_{kl} x_k y_l \pi_{kl})
\end{equation}

Since $\GG^{*}$ is supported on a permutation $\sigma^{*}$ this gives:
\begin{equation}
\tag{LP}
n \sum_{i} x_i^{2} y_{\sigma^{*}(i)}^{2}+ \max_{\GG \in DS} C(\xbf,\ybf,\sigma^{*})\sum_{kl} x_k y_l \pi_{kl}
\end{equation}
where $C(\xbf,\ybf,\sigma^{*})=2\left(\sum_{i}x_i y_{\sigma^{*}(i)}\right)$. 
\begin{itemize}
\item If $C(\xbf,\ybf,\sigma^{*})>0$ then this (LP) has a unique solution which is the identity $\GG^{*}=\mathbf{I_n}$. This is a consequence of the Rearrangement Inequality (see Theorem \ref{memo:rearr}) which states that for all permutations $\sum_{i} x_{i} y_{\sigma(i)} < \sum_{i} x_{i} y_{i}$ (since $x_i$ and $y_j$ are distinct). Using the fact that an optimal solution of (LP) is supported on a permutation concludes.
\item If $C(\xbf,\ybf,\sigma^{*})<0$ then the anti-identity is the unique optimum with the same reasoning since $\sum_i x_i y_{n+1-i} < \sum_{i} x_{i} y_{\sigma(i)}$ for all permutation because of Rearrangement Inequality.
\end{itemize}
\end{proof}

Using both results we can prove the following proposition which is the main ingredient to prove Theorem \ref{qap}:

\begin{prop}
\label{prop:twocases}
Let $\xbf,\ybf \in \mathcal{D}$ and $\sigma^{*}$ a solution of \eqref{eq:qap_to_prove} \ie\ $\sigma^{*} \in \argmax_{\sigma \in \Sn} g(\xbf,\ybf,\sigma)$. For $\sigma \in \Sn$ we note $C(\xbf,\ybf,\sigma)=\sum_{i}x_i y_{\sigma(i)}$. 

If $C(\xbf,\ybf,\sigma^{*})>0$ then $\sigma^{*}$ is the identiy permutation $\sigma^{*}(i)=i$ and if $C(\xbf,\ybf,\sigma^{*})<0$ then $\sigma^{*}$ is the anti-identity permutation $\sigma^{*}(i)=n+1-i$ for all $i \in \integ{n}$.
\end{prop}

\begin{proof}
Let $\sigma^{*}$ be an optimal solution of \eqref{eq:qap_to_prove} and $\GG^{*}$ defined by $\pi^{*}_{ij}=1$ if $j=\sigma^{*}(i)$ else $0$. By Lemma \ref{lemma:equivalence_both_probs} we know that $\GG^{*}$ is an optimal solution of \eqref{eq:qp_equiv}. Consider the case $C(\xbf,\ybf,\sigma^{*})>0$. Suppose that $\sigma^{*}$ is not the identity, then $\GG^{*} \neq \mathbf{I_n}$ which is not possible by Lemma \ref{lemma:gg_two_cases} since $\GG^{*}$ is an optimal solution of \eqref{eq:qp_equiv}. Same applies for $C(\xbf,\ybf,\sigma^{*})<0$ and the anti-identity.

\end{proof}

To state that \eqref{eq:qap_to_prove} admits the identity or the anti-identity as optimal permutations we will rely on the previous proposition and on the continuity of the loss $g$:

\begin{lemma}[Continuity of $g$]
\label{lemma:continuity_Z}
Let $\xbf,\ybf \in \mathcal{D}$ fixed. There exists $\epsilon_{x,y}>0$ such that for all $\|\mathbf{h}\|<\epsilon_{x,y}$ we have:
\begin{equation} 
\argmax_{\sigma \in \Sn} g(\xbf+\mathbf{h},\ybf,\sigma) \subset \argmax_{\sigma \in \Sn} g(\xbf,\ybf,\sigma)
\end{equation}
\end{lemma}
\begin{proof}
Let $\xbf,\ybf \in \mathcal D$, $\sigma^{*} \in \argmax_{\sigma \in \Sn} g(\xbf,\ybf,\sigma)$ and $\tau$ any permutation in $\Sn$ such that $\tau \notin \argmax_{\sigma \in \Sn} g(\xbf,\ybf,\sigma)$. Then we have $g(\xbf,\ybf,\sigma^{*})>g(\xbf,\ybf,\tau)$. Let $\eta=g(\xbf,\ybf,\sigma^{*})-g(\xbf,\ybf,\tau)>0$. For all permutation $\beta$ we have that $g(.,\ybf,\beta)$ is continuous. In this way:
\begin{equation}
\label{eq:continuityh}
\begin{split}
&\forall \beta \in \Sn, \exists \epsilon_{\xbf,\ybf}(\beta,\sigma^{*},\tau)>0, \ \forall \|\mathbf{h}\|<\epsilon_{\xbf,\ybf}(\beta,\sigma^{*},\tau), \ |g(\xbf+\mathbf{h},\ybf,\beta)-g(\xbf,\ybf,\beta)|< \frac{\eta}{4} \\
\end{split}
\end{equation}
Let $\mathbf{h}\in \R^{n}$ such that $\|\mathbf{h}\|<\underset{(\beta,\sigma,\tau') \in (\Sn)^{3}}{\min}\epsilon_{\xbf,\ybf}(\beta,\sigma,\tau')$. By \eqref{eq:continuityh} applied to $\sigma^{*}$ and $\tau$:
\begin{equation}
\begin{split}
g(\xbf+\mathbf{h},\ybf,\sigma^{*})-g(\xbf+\mathbf{h},\ybf,\tau)&=g(\xbf+\mathbf{h},\ybf,\sigma^{*})-g(\xbf,\ybf,\sigma^{*})\\
&+g(\xbf,\ybf,\sigma^{*})-g(\xbf,\ybf,\tau)+g(\xbf,\ybf,\tau)-g(\xbf+\mathbf{h},\ybf,\tau) \\
&>-\frac{\eta}{4}+\eta-\frac{\eta}{4} \\
&= \frac{\eta}{2}>0
\end{split}
\end{equation}
So that $g(\xbf+\mathbf{h},\ybf,\sigma^{*})>g(\xbf+\mathbf{h},\ybf,\tau)$ and in this way $\tau \notin \argmax_{\sigma \in \Sn} g(\xbf+\mathbf{h},\ybf,\sigma)$ because $\sigma^{*}$ leads to a striclty better cost. Overall we have proven that for any permutation $\tau$, if $\tau \notin \argmax_{\sigma \in \Sn} g(\xbf,\ybf,\sigma)$ and $\|\mathbf{h}\|<\underset{(\beta,\sigma,\tau') \in (\Sn)^{3}}{\min}\epsilon_{\xbf,\ybf}(\beta,\sigma,\tau')$ then $\tau \notin \argmax_{\sigma \in \Sn} g(\xbf+\mathbf{h},\ybf,\sigma)$ which proves that $\argmax_{\sigma \in \Sn} g(\xbf+\mathbf{h},\ybf,\sigma) \subset \argmax_{\sigma \in \Sn} g(\xbf,\ybf,\sigma)$.
\end{proof}

Using the previous lemma we can now prove the following result:

\begin{lemma}
\label{lemma:endlemma}
Let $\xbf,\ybf \in \mathcal{D}$ fixed. There exists $\epsilonb_{0} \in \R^{n}$ such that:
\begin{equation}
\begin{split}
&\argmax_{\sigma \in \Sn} g(\xbf+\epsilonb_{0},\ybf,\sigma) \subset \argmax_{\sigma \in \Sn} g(\xbf,\ybf,\sigma) \\
&\argmax_{\sigma \in \Sn} g(\xbf+\epsilonb_{0},\ybf,\sigma) \subset \{Id,anti-Id\}
\end{split}
\end{equation}
\end{lemma}

\begin{proof}
Let $\xbf,\ybf \in \mathcal{D}$. We consider $\epsilonb_{0}=(\zeta,-\zeta,0,\dots,0)$ with $\zeta >0$ small enough to ensure $\zeta<\frac{x_2-x_1}{2}$ and $\|\epsilonb_{0}\|< \epsilon_{x,y}$ defined in Lemma \ref{lemma:continuity_Z}. We have $\xbf+\epsilonb_0,\ybf \in \mathcal{D}$ since $\sum_{i} (x_i +\epsilon_0(i))=\sum_{i} x_i+ \zeta-\zeta = 0$ and $x_1+\epsilon_0(1)<\dots<x_N+\epsilon_0(N)$ since $x_1+\zeta < x_2 - \zeta$. 

Let $\sigma_{\epsilonb_{0}}^{*} \in \argmax_{\sigma \in \Sn} g(\xbf+\epsilonb_{0},\ybf,\sigma)$. By Lemma \ref{lemma:continuity_Z} we have $\sigma_{\epsilonb_{0}}^{*} \in \argmax_{\sigma \in \Sn} g(\xbf,\ybf,\sigma)$. 

Moreover we have $C(\xbf+\epsilonb_0,\ybf,\sigma_{\epsilonb_{0}}^{*})=\sum_{i}x_{i}y_{\sigma_{\epsilonb_{0}}^{*}(i)}+\zeta(y_{\sigma_{\epsilonb_{0}}^{*}(0)}-y_{\sigma_{\epsilonb_{0}}^{*}(1)})$.

\begin{itemize}
\item If $\sum_{i}x_{i}y_{\sigma_{\epsilonb_{0}}^{*}(i)}=0$ then $C(\xbf+\epsilonb_0,\ybf,\sigma_{\epsilonb_{0}}^{*})=\zeta(y_{\sigma_{\epsilonb_{0}}^{*}(0)}-y_{\sigma_{\epsilonb_{0}}^{*}(1)}) \neq 0$ since all $y_i$ are distinct. We can apply Proposition \ref{prop:twocases} with $\xbf+\epsilonb_0,\ybf \in \mathcal{D}$ to conclude that $\sigma_{\epsilonb_{0}}^{*}$ is wether the identity or the anti-identity.
\item If $\sum_{i}x_{i}y_{\sigma_{\epsilonb_{0}}^{*}(i)} \neq 0$ then $\sigma_{\epsilonb_{0}}^{*} \in \argmax_{\sigma \in \Sn} g(\xbf,\ybf,\sigma)$ and $C(\xbf,\ybf,\sigma_{\epsilonb_{0}}^{*}) \neq 0$ so by Proposition \ref{prop:twocases} with $\xbf,\ybf \in \mathcal{D}$ we can conclude that $\sigma_{\epsilonb_{0}}^{*}$ is wether the identity or the anti-identity.
\end{itemize}

\end{proof}

\begin{corr}[Theorem \ref{qap} is valid]
\label{prop:zerocase}
Let $\xbf,\ybf \in \mathcal{D}$. The identity or the anti-identity is an optimal solution of \eqref{eq:qap_to_prove}
\end{corr}

\begin{proof}
Let $\xbf,\ybf \in \mathcal{D}$. We consider $\epsilonb_{0}$ defined in Lemma \ref{lemma:endlemma} and $\sigma_{\epsilonb_{0}}^{*} \in \argmax_{\sigma \in \Sn} g(\xbf+\epsilonb_{0},\ybf,\sigma)$. Then by Lemma \ref{lemma:endlemma} $\sigma_{\epsilonb_{0}}^{*}$ is wether the identity or the anti-identity. Moreover by Lemma \ref{lemma:endlemma} $\sigma_{\epsilonb_{0}}^{*} \in \argmax_{\sigma \in \Sn} g(\xbf,\ybf,\sigma)$ so it is an optimal solution of \eqref{eq:qap_to_prove}. This concludes that the identity or the anti-identity is an optimal solution of \eqref{eq:qap_to_prove} which proves Theorem \ref{qap}.
\end{proof}

\section{Computing $GW$ in the 1d case}

As seen in the previous theorem finding the optimal permutation $\sigma^{*}$ can be found in $O(n\log(n))$. Moreover the final cost can be written using binomial expansion:

\begin{equation}
\begin{split}
\sum_{i,j} \big((x_{i}-x_{j})^{2} - (y_{\sigma^{*}(i)}-y_{\sigma^{*}(j)})^{2}\big)^{2}&=2n\sum_{i} x_{i}^{4} - 8 \sum_{i} x_{i}^{3} \sum_{i} x_{i} + 6 (\sum_{i} x_{i}^{2})^{2} \\
&+2n\sum_{i} y_{i}^{4}- 8 \sum_{i} y_{i}^{3} \sum_{i} y_{i} + 6 (\sum_{i} y_{i}^{2})^{2}\\
&-4(\sum_{i}x_{i})^{2}(\sum_{i}y_{i})^{2} \\
&-4n\sum_{i}  x_{i}^{2}y_{\sigma^{*}(i)}^{2}+ 8 \sum_{i}((\sum_{i} x_{i}) x_{i} y_{\sigma^{*}(i)}^{2} + (\sum_{i} y_{i}) x_{i}^{2} y_{\sigma^{*}(i)}) \\
&- 8(\sum_{i} x_{i} y_{\sigma^{*}(i)})^{2}
\end{split}
\label{eq:linear:computation}
\end{equation}

which can be computed in $O(n)$ operations.

\section{Claims about $GW$}

This section aims at proving some claims in the paper about $GW$. Let us recall the notations of the paper.

We consider discrete measures $\mu \in \Pm(\R^{p})$ and $\nu \in \Pm(\R^{q})$ with $p\leq q$ on euclidean spaces such that $\mu= \sum_{i=1}^{n} a_{i} \delta_{\xbf_{i}}$ and $\nu= \sum_{i=1}^{m} b_{j} \delta_{\ybf_{j}}$, where  $a \in \Sigma_{n}$ and $b \in \Sigma_{m}$ are histograms. 

Let $c_{X}: \R^{p} \times \R^{p} \mapsto \R_{+}$ (\textit{resp.} $c_{Y}: \R^{q} \times \R^{q} \mapsto \R_{+}$) measuring the similarity between the points in $\mu$ (\textit{resp.} $\nu$). 
The Gromov-Wasserstein ($\gw$) distance is defined as:

\begin{equation}
\label{gw}
\gw_{2}^{2}(c_{X},c_{Y},\mu,\nu)= \underset{\pi \in \couplingset(a,b)}{\min} J(c_{X},c_{Y},\pi)  \\
\end{equation}

where

\begin{equation*}
J(c_{X},c_{Y},\pi) =\sum_{i,j,k,l} \big| c_{X}(\xbf_{i},\xbf_{k})-c_{Y}(\ybf_{j},\ybf_{l}) \big|^{2} \pi_{i,j}\pi_{k,l}
\end{equation*}

\subsection{$GW$ when squared euclidean distances are used}

When $c_{X},c_{Y}$ are distances it has been shown in \cite{memoli_gw} that $GW$ defines a distance on the space of metric measure spaces quotiented by the measure-preserving isometries. More precisely, $GW$ is symmetric, satisfies the triangle inequality and $\gw_{2}^{2}(c_{X},c_{Y},\mu,\nu)=0$ \textit{iff} there exists $f: \text{supp}(\mu) \rightarrow \text{supp}(\nu)$ such that 
\begin{equation}
\label{conservation_gw}
f\#\mu=\nu
\end{equation}
\begin{equation}
\label{isometrygw}
\forall \xbf,\xbf' \in \text{supp}(\mu)^{2}, c_{X}(\xbf,\xbf')=c_{Y}(f(\xbf),f(\xbf'))
\end{equation}

In the paper we claim the following lemma:
\begin{lemma}{} 
\label{gw_with_squared_distances}
Using previous notations with $c_{X}(\xbf,\xbf')=\|\xbf-\xbf'\|_{2,p}^{2}$ , $c_{Y}(\ybf,\ybf')=\|\ybf-\ybf'\|_{2,q}^{2}$. Then $\gw_{2}^{2}(c_{X},c_{Y},\mu,\nu)=0$ \textit{iff} there exists a measure preserving isometry from $\text{supp}(\mu)$ to $\text{supp}(\nu)$ which satisfies \eqref{conservation_gw} and \eqref{isometrygw}
\end{lemma}

\begin{proof}
If such an function exists by considering the coupling $\pi=(I_{d} \times f) \# \mu$ it is clear that $\pi$ is optimal and has a zero cost so that $\gw_{2}^{2}(c_{X},c_{Y},\mu,\nu)=0$. 
Conversely, $\gw_{2}^{2}(c_{X},c_{Y},\mu,\nu)=0$ implies that $c_{X}(\xbf,\xbf')=c_{Y}(\ybf,\ybf')$ for all $(\xbf,\ybf),(\xbf',\ybf')$ in the support of an optimal plan $\pi^{*}$. This suffices to prove the existence of a measure preserving isometry (see (a) in Proof of Theorem 5.1 in \cite{memoli_gw})
\end{proof}

\subsection{Equivalence between $GM$ and $GW$ in the discrete case}

This paragraph aims at proving the equivalence between $GM$ and $GW$. We will prove the following theorem (that is more general than the one used in the paper which only considers one-dimensional measures): 

\begin{theorem}{Equivalence between $\gw$ and $\gm$ for discrete measures}
\label{sovable_gw}

With $\mu$, $\nu$ defined previously and $c_{X}(\xbf,\xbf')=\|\xbf-\xbf'\|_{2,p}^{2}$ , $c_{Y}(\ybf,\ybf')=\|\ybf-\ybf'\|_{2,q}^{2}$. Let us suppose also that $m=n$ and $\forall i \in \{1,...,n\}, a_{i}=b_{i}=\frac{1}{n}$

Then $\gw_{2}(c_{X},c_{Y},\mu,\nu)=\gm_{2}(c_{X},c_{Y},\mu,\nu)$. 
\end{theorem}

\begin{proof}
The proof is essentially based on theoretical results from \cite{NIPS2018_7323} and on Theorem \ref{qap}. In \cite{NIPS2018_7323} authors consider the minimizing energy problem $\underset{\X \in \Pi_n}{\min} -\tr(\mathbf{B}\X^{T}\mathbf{A}\X)$ where $\Pi_n$ the set of permutation matrices. In fact, the $GM$ problem defined in this chapter is equivalent to $\underset{\X \in \Pi_n}{\min} -\tr(\mathbf{B}\X^{T}\mathbf{A}\X)$ by considering $\mathbf{A}=(\|\xbf_{i}-\xbf_{j}\|_{2,p}^{2})_{i,j}$ and $\mathbf{B}=(\|\ybf_{i}-\ybf_{j}\|_{2,q}^{2})_{i,j}$.

To tackle this problem authors propose to minimize $-\tr(\mathbf{B}\X^{T}\mathbf{A}\X)$ over the set of doubly stochastic matrices (which is the convex-hull of $\Pi_n$): 
\begin{equation*}
DS=\{\X\in \mathbb{R}^{n\times n} \ \text{s.t.} \ \X \one_n=\X^{T}\one_n=\one_n \ , \X \geq 0\}
\end{equation*}
Minimizing $-\tr(\mathbf{B}\X^{T}\mathbf{A}\X)$ over $DS$ is equivalent to solving the $GW$ distance when $a_{i}=b_{j}=\frac{1}{n}$. The paper proves that when both $\mathbf{A}$ and $\mathbf{B}$ are conditionally positive (or negative) definite of order 1 then the relaxation leads to the same optimum so that the minimum over $DS$ is the same as the minimum over $\Pi_n$ \cite[Theorem 1]{NIPS2018_7323}. Yet $\mathbf{A}$ and $\mathbf{B}$ defined previously satisfy this property (see examples under Definition 2 in \cite{NIPS2018_7323}) and so $GW$ and $GM$ coincide.

Moreover when $p=q=1$ and when the sample are sorted we can apply Theorem \ref{qap} to prove that an optimal permutation of the $GM$ problem is found whether at the identity or the anti-identity permutation which concludes the proof.

\end{proof}

\section{Properties of \sgw}

$\|.\|$ is a norm on $\mathbb{R}^{p}$. To state the properties of $\sgw$, we will need the Arzela-Ascoli Theorem. Let $(X,d)$ be a compact metric space and $C(X,\mathbb{R}^{p})$ the space of all continuous functions from $X$ to $\mathbb{R}^{p}$. We recall:

\begin{itemize}
\item A family $\mathcal{F} \subset C(X,\mathbb{R}^{p})$  is \emph{bounded} means that there exists a positive constant $M<\infty$ such that $\|f(x)\| \leq M$ for all $x \in X$ and  $f \in \mathcal{F}$
\item A family $\mathcal{F} \subset C(X,\mathbb{R}^{p})$  is \emph{equicontinuous} means that for every  $\epsilon>0$ there exists $\delta>0$ (which depends only on  $\epsilon$) such that for  $x, y \in X$: $$d(x, y)<\delta \Rightarrow\|f(x)-f(y)\|<\epsilon \quad \forall f \in \mathcal{F}.$$ 
\end{itemize}

The Arzela-Ascoli Theorem states that if $(f_{n})_{n\in \mathbb{N}}$ is a sequence in $C(X,\mathbb{R}^{p})$ that is bounded and equicontinuous then it has a uniformly convergent subsequence.

We recall the theorem (measures $\mu$ and $\nu$ are defined discrete measures with the same number of atoms):

\begin{theorem}{Properties of $\sgw$}
\label{propertiessgw}
\begin{itemize}
\item For all $\D$, $\sgw_{\D}$ and $\risgw$ are translation invariant. $\risgw$ is also rotational invariant when $p=q$, more precisely if $\Qbf \in \mathcal{O}(p)$ is an orthogonal matrix, $\risgw(\Qbf\#\mu,\nu)=\risgw(\mu,\nu)$ (same for any $\Qbf' \in \mathcal{O}(p)$ applied on $\nu$).
\item $\sgw$ and $\risgw$ are pseudo-distances on $\Pm(\R^{p})$, \textit{i.e.} they are symmetric, satisfy the triangle inequality and $\sgw(\mu,\mu)=\risgw(\mu,\mu)=0$ .
\item Let $\mu,\nu \in \Pm(\R^{p})\times \Pm(\R^{p})$ be probability distributions with \emph{compact supports}. If $\sgw(\mu,\nu)=0$ then $\mu$ and $\nu$ are isomorphic for the distance induced by the $\ell_{1}$ norm on $\R^{p}$, \ie\ $d(\xbf,\xbf')=\sum_{i=1}^{p} |x_{i}-x_{i}'|$ for $(\xbf,\xbf') \in \R^{p} \times \R^{p}$. In particular this implies:
\begin{equation}
\sgw(\mu,\nu)=0 \implies \gw_{2}(d,\mu,\nu)=0
\end{equation} 
\end{itemize}
\end{theorem}

The invariance by translation is clear since the costs are invariant by translation of the support of the measures. The pseudo-distances properties are straightforward thanks to the properties of $\gw$. For the invariance by rotation if  $p=q$ then $\mathbb{V}_{p}(\R^{p})$ is bijective with $\mathcal{O}(p)$ so for $\Qbf \in \mathcal{O}(p)$:

\begin{equation}
\begin{split}
\risgw(\Qbf\#\mu,\nu)&=\underset{\D \in \mathbb{V}_{p}(\R^{p})}{\min}\sgw_{\D}(\Qbf\#\mu,\nu) \\
&=\underset{\D \in \mathcal{O}(p)}{\min}\sgw_{\D}(\Qbf\#\mu,\nu) \\
&= \underset{\D \in \mathcal{O}(p)}{\min} \underset{\thetab \sim \lambda_{q-1}}{\E}[\gw(d^{2},P_{\thetab}\#(\D \Qbf\#\mu),P_{\thetab}\#\nu)] \\
&= \underset{\D' \in \mathcal{O}(p)}{\min} \underset{\theta \sim \lambda_{q-1}}{\E}[\gw(d^{2},P_{\thetab}\#\D'\#\mu,P_{\thetab}\#\nu)] \\
&= \risgw(\mu,\nu)
\end{split}
\end{equation}

On the other side for $\nu$ a change of formula on theta gives the result.

For the case $SGW=0 \implies GW=0$ it will be a consequence of the following theorem:

\begin{theorem}[]
\label{cramer_gene}
Let $\mu,\nu \in \Pm(\R^{p})\times \Pm(\R^{p})$ be probability distributions such that $\mu,\nu$ have compact supports. If for almost all $\thetab \in \Sp^{p-1}$, $P_{\thetab}\#\mu$ and $P_{\thetab}\#\nu$ are isomoprhic then $\mu$ and $\nu$ are isomorphic. In other words if for almost all $\thetab \in \Sp^{p-1}$ we have:
\begin{equation}
\begin{split}
&\exists T_{\thetab} : \supp(P_{\thetab}\#\mu) \subset \R \mapsto \supp(P_{\thetab}\#\nu) \subset \R,  \ \text{surjective} \ \text{s.t.} \ T_{\thetab} \# (P_{\thetab}\#\mu)= P_{\thetab}\#\nu \\
&\forall x,x' \in \supp(P_{\thetab}\#\mu), |T_{\thetab}(x)-T_{\thetab}(x')|=|x-x'|
\end{split}
\end{equation}
Then there exists a measure preserving isometry $f$ between $\supp(\mu)$ and $\supp(\nu)$. More precisely we have $f\#\mu=\nu$ and: 
\begin{equation}
\forall \xbf,\xbf' \in \text{supp}(\mu), \|f(\xbf)-f(\xbf')\|_1=\|\xbf-\xbf'\|_1
\end{equation}
\end{theorem}

To prove this theorem we will exhibit the isometry. This result can be put in light of Cramer–Wold theorem \cite{cramer} which states that a probability measure is uniquely determined by the totality of its one-dimensional projections. Equivalently, if we consider two probability measures so that the one-dimensional measures resulting from the projections over all the hypersphere are equal then the measures are equal. The equality relation is replaced in our theorem by the isomoprhism relation.

The proof is divided into four parts. In the first one, we construct an "almost orthogonal" basis on which measures are isomorphic. Building upon this result we define a sequence of functions from $\text{supp}(\mu)$ to $\text{supp}(\nu)$ and show that it has a convergent subsequence. We conclude by proving that the limit of the subsequence is actually a good candidate for being the isometry we are looking for. In the following $\|.\|_{1}$ denotes the $\ell_{1}$ norm, $\|.\|_{2}$ denotes the $\ell_{2}$ norm and $p\geq 2$. We recall that $\mathcal{F}_{\mu}$ is the Fourier transform of $\mu$. 

\smallskip
We consider the following $\mathcal{Q}_{\thetab}$ property for $\thetab \in \Sp^{p-1}$:
\begin{equation}\tag{$\mathcal{Q}_{\thetab}$}
\label{conserv}
\begin{split}
&\exists T_{\thetab} : \supp(P_{\thetab}\#\mu) \subset \R \mapsto \supp(P_{\thetab}\#\nu) \subset \R,  \ \text{surjective} \ \text{s.t.} \ T_{\thetab} \# (P_{\thetab}\#\mu)= P_{\thetab}\#\nu \\
&\forall x,x' \in \text{supp}(P_{\thetab}\#\mu), |T_{\thetab}(x)-T_{\thetab}(x')|=|x-x'|
\end{split}
\end{equation}
Informally if we have the $\mathcal{Q}_{\thetab}$ property for $\thetab \in \Sp^{p-1}$ it implies that $\mu$ and $\nu$ are isomorphic on the 1D line given by the projection \textit{w.r.t.} $\thetab$. We have the following lemma:
\begin{lemma}
\label{allmost_orthogonal_basis}
Let $\mu,\nu \in \Pm(\R^{p})\times \Pm(\R^{p})$ and suppose that $\mathcal{Q}_{\thetab}$ holds for almost all $\thetab \in \Sp^{p-1}$. Let $n>p-1$. There exists a basis $(\ebf_{1}(n),...,\ebf_{p}(n))$ of $\R^{p}$ part of the following spaces:

\begin{equation}
\mathcal{B}^{n}_{p}\stackrel{def}{=}\{(\thetab_{1},...,\thetab_{p}) \in (\Sp^{p-1})^{p} \ \text{s.t.} \ |\langle \thetab_{i},\thetab_{j} \rangle | < \frac{1}{n}\}
\end{equation}
and 
\begin{equation}
Q\stackrel{def}{=}\{(\thetab_{1},...,\thetab_{p}) \in (\Sp^{p-1})^{p} \ \text{s.t.} \ \forall i \in \{1,...,p\}, \mathcal{Q}_{\thetab_{i}} \}
\end{equation}

\end{lemma}

\begin{proof}
We want to construct a basis $(\ebf_{1},...,\ebf_{p})$ as orthogonal as possible such that for all $i$ we have $\mathcal{Q}_{\ebf_{i}}$.

We note $\lambda^{\otimes p}_{p-1}$ the product measure $\lambda_{p-1}\otimes ... \otimes \lambda_{p-1}$ where $\lambda_{p-1}$ is the uniform measure on the sphere. $\mathcal{B}^{n}_{p}$ is an open set as inverse image by a continuous function of an open set. Then $\lambda^{\otimes p}_{p-1}(\mathcal{B}^{n}_{p})>0$. Moreover, since for almost all 
$\thetab \in \Sp^{p-1}$ we have $\mathcal{Q}_{\thetab}$ then $\lambda^{\otimes p}_{p-1}(Q)>0$ and so $\lambda^{\otimes p}_{p-1}(\mathcal{B}^{n}_{p} \cap Q) >0$. 

In this way we can consider $(\ebf_{1}(n),...,\ebf_{p}(n)) \in \mathcal{B}^{n}_{p} \cap Q$. If $n>p-1$ the Gram matrix of $(\ebf_{1}(n),...,\ebf_{p}(n))$ is strictly diagonal dominant, thus invertible, such that $(\ebf_{1}(n),...,\ebf_{p}(n))$ is a basis. Note that we can not consider directly an orthogonal basis since the set of all orthogonal basis has measure zero. 
\end{proof}

We now express all the vectors and inner products in this new almost orthogonal basis as expressed in the following lemma:

\begin{lemma}
\label{write_in_new}
Let $n>p-1$ and a basis $(\ebf_{1}(n),...,\ebf_{p}(n))$ as defined in Lemma \ref{allmost_orthogonal_basis}. Then all $\xbf \in \R^{p}$ can be written as:
\begin{equation}
\xbf= \sum_{i=1}^{p} [\langle \xbf,\ebf_{i}(n) \rangle+ R(\xbf,\ebf_{i}(n))] \ebf_{i}(n)
\end{equation}
where $|R(\xbf,\ebf_{i}(n))|=o(\frac{1}{n})$. Moreover for all $(\xbf,\ybf) \in \R^{p} \times \R^{p}$:
\begin{equation}
\langle \xbf,\ybf\rangle=\sum_{i=1}^{p} \langle \xbf,\ebf_{i}(n)\rangle \langle \ybf,\ebf_{i}(n)\rangle+\tilde{R}(\xbf,\ybf)
\end{equation} 
where $|\tilde{R}(\xbf,\ybf)|=o(\frac{1}{n})$.
\end{lemma}

\begin{proof}
In the following $x_i$ denotes the $i$-th coordinate of a vector $\xbf$ in the standard basis, \ie\ a vector writes $\xbf=(x_1,\dots,x_p)$. For $\xbf\in \R^{p}$, we can write in the new basis $\xbf= \sum_{i=1}^{p} [\langle \xbf,\ebf_{i}(n) \rangle+ R(\xbf,\ebf_{i}(n))] \ebf_{i}$ with $R(\xbf,\ebf_{i}(n))\stackrel{def}{=}x_{i}-\langle \xbf,\ebf_{i}(n) \rangle$. We have also $|R(\xbf,\ebf_{i}(n))|=o(\frac{1}{n})$. Indeed, 
\begin{equation*}
\begin{split}
\xbf=\sum_{i=1}^{p} x_i\ebf_{i} &\implies \forall j, \langle \xbf, \ebf_{j} \rangle= \sum_{i=1}^{p} x_{i} \langle \ebf_{i},\ebf_{j}\rangle \implies x_{j} - \langle \xbf, \ebf_{j} \rangle =\sum_{i \neq j} x_{i} \langle \ebf_{i},\ebf_{j}\rangle \\
&\implies |R(\xbf,\ebf_{j}(n))| = | \sum_{i \neq j} x_{i} \langle \ebf_{i},\ebf_{j}\rangle |  \implies |R(\xbf,\ebf_{j}(n))| \leq \frac{1}{n} \sum_{i \neq j} |x_{i}| 
\end{split}
\end{equation*}
Also in the same way for $\xbf,\ybf \in \R^{p} \times \R^{p}$ we can rewrite their inner product:
\begin{equation}
\label{doublescalr}
\langle \xbf,\ybf\rangle=\sum_{i=1}^{p} \langle \xbf,\ebf_{i}(n)\rangle \langle \ybf,\ebf_{i}(n)\rangle+\tilde{R}(\xbf,\ybf)
\end{equation} 
with:
\begin{equation*}
\begin{split}
\tilde{R}(\xbf,\ybf)&\stackrel{def}{=}\langle \xbf,\ybf\rangle - \sum_{i=1}^{p} \langle \xbf,\ebf_{i}(n)\rangle \langle \ybf,\ebf_{i}(n)\rangle \\
&= \sum_{i\neq j} \langle \xbf,\ebf_{i}(n)\rangle \langle \ybf,\ebf_{i}(n)\rangle \langle \ebf_{j}(n),\ebf_{i}(n)\rangle + \sum_{i,j}  \langle \xbf,\ebf_{i}(n)\rangle R(\ybf,\ebf_{j}(n)) \langle \ebf_{j}(n),\ebf_{i}(n)\rangle \\
&+  \sum_{i,j}  \langle \ybf,\ebf_{j}(n)\rangle R(\xbf,\ebf_{i}(n)) \langle \ebf_{j}(n),\ebf_{i}(n)\rangle + \sum_{i,j} R(\xbf,\ebf_{j}(n)) R(\ybf,\ebf_{i}(n))\langle \ebf_{j}(n),\ebf_{i}(n)\rangle 
\end{split}
\end{equation*}
and with the same calculus than for $R$ we have $|\tilde{R}(\xbf,\ybf)|=o(\frac{1}{n})$.
\end{proof}

\begin{prop}
\label{thebig_prop}
Let $\mu,\nu \in \Pm(\R^{p})\times \Pm(\R^{p})$ and suppose that $\mathcal{Q}_{\thetab}$ holds for almost all $\thetab \in \Sp^{p-1}$ and that $\nu$ has compact support. There exists a sequence $(f_{n})_{n\in \mathbb{N}}$ from $\supp(\mu)$ to $\supp(\nu)$ uniformly bounded which satisfies:

\begin{equation}
\label{close}
\forall n \in \mathbb{N}, \forall \xbf,\xbf' \in \supp(\mu)^{2}, \ \big| \|f_{n}(\xbf)-f_{n}(\xbf')\|_1-\|\xbf-\xbf'\|_1 \big| =  o(\frac{1}{n})
\end{equation}

\begin{equation}
\label{fourr}
\forall n \in \mathbb{N}, \forall \sbf \in \R^{p}, \ |\mathcal{F}_{f_{n}\#\mu}(\sbf)-\mathcal{F}_{\nu}(\sbf)|=o(\frac{1}{n})
\end{equation}

\end{prop}

\begin{proof}
In the following $x_i$ denotes the $i$-th coordinate of a vector $\xbf$ in the standard basis, \ie\ a vector writes $\xbf=(x_1,\dots,x_p)$. We define:
\begin{equation}
\label{fn}
\forall n >p-1, \ \forall \xbf \in \text{supp}(\mu), \ f_{n}(\xbf)=(T_{\ebf_{1}(n)}(\langle \xbf,\ebf_{1}(n)\rangle),...,T_{\ebf_{p}(n)}(\langle \xbf,\ebf_{p}(n)\rangle))
\end{equation}
where $(\ebf_k(n))_{k\in \integ{p}}$ is the almost orthogonal basis define in Lemma \ref{allmost_orthogonal_basis}, and $T_{\ebf_{k}(n)}$ is defined from \eqref{conserv} since we have $\mathcal{Q}_{\ebf_{k}(n)}$ for all $k$. It is clear from the definition that $f_{n}(\xbf) \in \supp(\nu)$. Moreover for $\xbf,\xbf' \in \supp(\mu)$:
\begin{equation*}
\begin{split}
\|f_{n}(\xbf)-f_{n}(\xbf')\|_1&= \sum_{k=1}^{p} |T_{\ebf_{k}(n)}(\langle \xbf,\ebf_{k}(n)\rangle)-T_{\ebf_{k}(n)}(\langle \xbf',\ebf_{k}(n)\rangle) | \stackrel{(*)}{=} \sum_{k=1}^{p} |\langle \xbf,\ebf_{k}(n)\rangle-\langle \xbf',\ebf_{k}(n)\rangle| \\
&= \sum_{k=1}^{p} |\langle \xbf-\xbf',\ebf_{k}(n)\rangle| \\
\end{split}
\end{equation*}
where in (*) we used that $T_{\ebf_{k}(n)}$ is an isometry since we have $\mathcal{Q}_{\ebf_{k}(n)}$ and $\langle \xbf,\ebf_{k}(n)\rangle \in \supp(P_{\ebf_{k}(n)} \# \mu)$ (idem for $\xbf'$). 
In this way: 
\begin{equation*}
\begin{split}
\big| \|f_{n}(\xbf)-f_{n}(\xbf')\|_1-\|\xbf-\xbf'\|_1 \big|&= \big| \sum_{k=1}^{p} |\langle \xbf-\xbf',\ebf_{k}(n)\rangle| -|x_{k}-x_{k}'|  \big| \leq \sum_{k=1}^{p}\big| |\langle \xbf-\xbf',\ebf_{k}(n)\rangle| -|x_{k}-x_{k}'| \big| \\
&\stackrel{*}{\leq} \sum_{k=1}^{p}|\langle \xbf-\xbf',\ebf_{k}(n)\rangle -(x_{k}-x_{k}')| = \sum_{k=1}^{p}|R(\xbf-\xbf',\ebf_{k}(n))|=o(\frac{1}{n}) \\
\end{split}
\end{equation*}
where in (*) the second triangular inequality $| |x|- |y| |\leq |x-y|$. Hence: 
\begin{equation}
\big| \|f_{n}(\xbf)-f_{n}(\xbf')\|_1-\|\xbf-\xbf'\|_1 \big| =o(\frac{1}{n})
\end{equation}
Moreover we have by definition of the Fourier transform, for $s \in \R^{P}$, 
\begin{equation*}
\label{batmanbegin}
\begin{split}
\mathcal{F}_{f_{n}\#\mu}(\sbf)&= \int e^{-2i\pi\langle \sbf,f_{n}(\xbf)\rangle} d\mu(\xbf) = \int e^{-2i\pi \sum_{k=1}^{p} s_{k} T_{\ebf_{k}(n)}(\langle \xbf,\ebf_{k}(n)\rangle)} d\mu(\xbf) 
\end{split}
\end{equation*}
Moreover using \eqref{conserv} we have $\mathcal{F}_{T_{\ebf_{k}(n)}\#(P_{\ebf_{k}(n)}\#\mu)}(t)=\mathcal{F}_{P_{\ebf_{k}(n)}\#\nu}(t)$ for all $k\in \{1,...,p\}$, and any real $t\in \R$. This implies $\int e^{-2i\pi t.T_{\ebf_{k}(n)}(\langle \ebf_{k}(n) ,\xbf\rangle)} \dr\mu(\xbf) =  \int e^{-2i\pi t \langle \ebf_{k}(n) ,\ybf\rangle} \dr\nu(\ybf)$. So by applying this results for $t=s_{k}$ we have: 
\begin{equation}
\label{equivfourr2}
\int e^{-2i\pi  s_{k}  T_{\ebf_{k}(n)}(\langle \xbf,\ebf_{k}(n)\rangle)} d\mu(\xbf) =  \int e^{-2i\pi s_{k}  \langle \ebf_{k}(n),y\rangle} d\nu(\ybf)
\end{equation}
Combining both results: 
\begin{equation}
\label{fourr_}
\mathcal{F}_{f_{n}\#\mu}(\sbf)=  \int e^{ -2i\pi \sum_{k=1}^{p} s_{k}  \langle \ebf_{k}(n),\ybf\rangle} \dr\nu(\ybf)
\end{equation}
We can now bound $|\mathcal{F}_{f_{n}\#\mu}(\sbf)-\mathcal{F}_{\nu}(\sbf)|$ as:
\begin{equation*}
\begin{split}
|\mathcal{F}_{f_{n}\#\mu}(\sbf)-\mathcal{F}_{\nu}(\sbf)|&= | \mathcal{F}_{f_{n}\#\mu}(\sbf)- \int e^{-2i\pi\langle \sbf,\ybf\rangle} \dr\nu(\ybf) |\\
&\stackrel{*}{=} | \mathcal{F}_{f_{n}\#\mu}(\sbf) - \int e^{-2i\pi [\sum_{k=1}^{p} \langle \sbf,\ebf_{k}(n) \rangle \langle \ebf_{k}(n),\ybf\rangle +\tilde{R}(\sbf,\ybf)]} \dr\nu(\ybf) | \\
&\stackrel{**}{=} |  \int e^{-2i\pi \sum_{k=1}^{p} s_{k}  \langle \ebf_{k}(n),\ybf\rangle} \dr\nu(\ybf) - \int e^{-2i\pi\tilde{R}(\sbf,\ybf)}e^{-2i\pi \sum_{k=1}^{p} \langle \sbf,\ebf_{k}(n) \rangle \langle \ebf_{k}(n),\ybf\rangle } \dr\nu(\ybf) | \\
\end{split}
\end{equation*}
where in (*) we used the expression in the new base of the inner product $\langle \sbf,\ybf\rangle$ seen in Lemma \ref{write_in_new}, in (**) we used \eqref{fourr_}. By injecting the expression of $s_{k}$ \textit{w.r.t.} the new base we have:
\begin{equation}
\begin{split}
|\mathcal{F}_{f_{n}\#\mu}(\sbf)-\mathcal{F}_{\nu}(\sbf)|&\leq |  \int e^{-2i\pi \sum_{k=1}^{p} (\langle \sbf,\ebf_{k}(n) \rangle +R(\sbf,\ebf_{k}(n)))  \langle \ebf_{k}(n),\ybf\rangle} \dr\nu(\ybf) \\
&- \int e^{-2i\pi\tilde{R}(\sbf,\ybf)}e^{-2i\pi \sum_{k=1}^{p} \langle s,\ebf_{k}(n) \rangle \langle \ebf_{k}(n),y\rangle } \dr\nu(\ybf) | \\
&=\big| \int e^{-2i\pi \sum_{k=1}^{p} \langle \sbf,\ebf_{k}(n) \rangle \langle \ebf_{k}(n),\ybf\rangle} (e^{-2i\pi \sum_{k=1}^{p} R(\sbf,\ebf_{k}(n)) \langle \ebf_{k}(n),\ybf\rangle}-e^{-2i\pi\tilde{R}(\sbf,\ybf)}) \dr\nu(\ybf) \big| \\
&\leq \int |e^{-2i\pi \sum_{k=1}^{p} R(\sbf,\ebf_{k}(n)) \langle \ebf_{k}(n),\ybf\rangle}-e^{-2i\pi\tilde{R}(\sbf,\ybf)}| \dr\nu(\ybf) \\
&= \int |e^{-2i\pi\tilde{R}(\sbf,\ybf)} (e^{-2i\pi (\sum_{k=1}^{p} R(\sbf,\ebf_{k}(n)) \langle \ebf_{k}(n),y\rangle-\tilde{R}(\sbf,\ybf))}-1)|\dr\nu(\ybf) \\
&\leq \int |e^{-2i\pi (\sum_{k=1}^{p} R(\sbf,\ebf_{k}(n)) \langle \ebf_{k}(n),\ybf\rangle-\tilde{R}(\sbf,\ybf))}-1| \dr\nu(\ybf) \\
&= \int | 2ie^{-i\pi (\sum_{k=1}^{p} R(\sbf,\ebf_{k}(n)) \langle \ebf_{k}(n),\ybf\rangle-\tilde{R}(\sbf,\ybf))} \sin(\pi (\sum_{k=1}^{p} R(\sbf,\ebf_{k}(n)) \langle \ebf_{k}(n),\ybf\rangle-\tilde{R}(\sbf,\ybf)) | \dr\nu(\ybf) \\
&\leq  \int | \sin(\pi (\sum_{k=1}^{p} R(\sbf,\ebf_{k}(n)) \langle \ebf_{k}(n),\ybf\rangle-\tilde{R}(\sbf,\ybf)) | \dr\nu(\ybf) \\
&\leq \pi \int ( \sum_{k=1}^{p} |R(\sbf,\ebf_{k}(n)) \langle \ebf_{k}(n),\ybf\rangle|+|\tilde{R}(\sbf,\ybf)| )\dr\nu(\ybf) \\
&\stackrel{*}{=}o(\frac{1}{n})
\end{split}
\end{equation}
in (*) the fact that each term is $o(\frac{1}{n})$. In this way: 
\begin{equation}
|\mathcal{F}_{f_{n}\#\mu}(\sbf)-\mathcal{F}_{\nu}(\sbf)| =o(\frac{1}{n})
\end{equation}

Moreover $(f_{n})_{n > p-1}$ is also uniformly bounded. To see that we consider $\xbf \in \text{supp}(\mu)$. We have that for all $k \in \integ{p}$ $T_{\ebf_{k}(n)}(\langle \xbf,e_{k}(n)\rangle) \in \supp(P_{\ebf_{k}(n)}\#\nu)$ by definition of $T_{\ebf_{k}(n)}$. So there exists a $\ybf_{0}(\xbf,n,k) \in \text{supp}(\nu)$ such that $T_{\ebf_{k}(n)}(\langle \xbf,\ebf_{k}(n)\rangle)=\langle \ybf_{0}(x,n,k),\ebf_{k}(n)\rangle$. In this way $|T_{\ebf_{k}(n)}(\langle \xbf,\ebf_{k}(n)\rangle)|=|\langle \ybf_{0}(x,n,k),\ebf_{k}(n)\rangle| \leq \|\ybf_{0}(x,n,k)\|_{2} \|\ebf_{k}(n)\|_2$ by Cauchy-Swartz. 

Moreover $\|\ebf_{k}(n)\|_2 < \sqrt{\frac{1}{n}} \leq \sqrt{\frac{1}{p-1}} \leq 1$ and since $\nu$ has compact support then there is a constant $M_{\nu}$ we have $\|\ybf_{0}(x,n,k)\|_{2}\leq M_{\nu}$

So we have for $n\in \mathbb{N}$, $\xbf \in \text{supp}(\mu)$, 
\begin{equation*}
 \begin{split}
\|f_{n}(\xbf)\|_{2}^{2}&=\sum_{k=1}^{p} |T_{\ebf_{k}(n)}(\langle \xbf,\ebf_{k}(n)\rangle)|^{2} \leq p M_{\nu}
\end{split}
\end{equation*}
Since on $\mathbb{R}^{p}$ all norms are equivalent this suffices to state the existence of a constant $C$ such that $\forall \xbf \in \mathbb{R}^{p}, n \in \mathbb{N}, \|f_{n}(\xbf)\|_1\leq C$ so that $(f_n)_{n\in \mathbb{N}}$ is uniformly bounded. Reindexing $(f_n)_{n> p-1}$ gives the desired result.

\end{proof}

We can now prove Theorem \ref{cramer_gene}.

\begin{proof}[Proof of Theorem \ref{cramer_gene}]
We consider the sequence $(f_{n})_{n \in \mathbb{N}}$ defined in Proposition \ref{thebig_prop}. We will show that $(f_{n})_{n \in \mathbb{N}}$ is equicontinuous. Let $\epsilon >0$, using \eqref{close} there exists a $N\in \mathbb{N}$ such that  we have for all $\xbf,\xbf' \in \text{supp}(\mu)$:
\begin{equation} 
\|f_{n}(\xbf)-f_{n}(\xbf')\|_1 \leq  \epsilon +\|\xbf-\xbf'\|_1 \ \text{ for all } n \geq N
\end{equation}

Now let $\delta < \epsilon$. Suppose that $\|\xbf-\xbf'\|_1 < \delta$ then 

\begin{equation} 
\|f_{n}(x)-f_{n}(x')\|_1 <  \epsilon +\delta < 2 \epsilon \ \text{ for all } n \geq N
\end{equation}

Without loss of generality we can reindex $(f_{n})_{n \in \mathbb{N}}$ for $n$ large enough ($n\geq N$) so that $(f_{n})_{n \in \mathbb{N}}$ is equicontinuous with the previous argument.

Since $(f_{n})_{n \in \mathbb{N}}$ is a uniformly bounded and equicontinuous sequence from the support of $\mu$ which is compact to $\R^{p}$ we can apply Arzela-Ascoli theorem which states that $(f_{n})_{n \in \mathbb{N}}$ has a uniformly convergent subsequence. We denote by $(f_{\phi(n)})_n$ this sequence. We have $f_{\phi(n)} \underset{n \to \infty}{\underset{\rightarrow}{u}} f$ this sequence.

Moreover equation \eqref{fourr} states that for all $\sbf \in \R^{p}$, $\mathcal{F}_{f_{n}\#\mu}(\sbf) \underset{n \to \infty}{\rightarrow} \mathcal{F}_{\nu}(\sbf)$. In this way $(\mathcal{F}_{f_{n}\#\mu}(\sbf))_{n \in \mathbb{N}}$ is a convergent real valued sequence, so every adherence value goes to the same limit, hence $\mathcal{F}_{f_{\phi(n)}\#\mu}(\sbf) \underset{n \to \infty}{\rightarrow} \mathcal{F}_{\nu}(\sbf)$.

Moreover the function $f$ is a measure preserving isometry from $\supp(\mu)$ to $\supp(\nu)$. Indeed let $\epsilon_{1} >0, \sbf\in \R^{p}$, there exists from previous statements $N_{0},N_{1} \in \mathbb{N}$ such that for $n\geq N_{0}$, $|\mathcal{F}_{f_{\phi(n)}\#\mu}(\sbf)-\mathcal{F}_{\nu}(\sbf)| < \epsilon_{1}$ and $n\geq N_{1}$, $|\mathcal{F}_{f_{\phi(n)}\#\mu}(\sbf)-\mathcal{F}_{f\#\mu}(\sbf)| < \epsilon_{1}$. Let $n \geq \text{max}(N_{0},N_{1})$
\begin{equation*}
\begin{split}
|\mathcal{F}_{f\#\mu}(\sbf)-\mathcal{F}_{\nu}(\sbf)| &\leq |\mathcal{F}_{f_{\phi(n)}\#\mu}(\sbf)-\mathcal{F}_{\nu}(\sbf)| + |\mathcal{F}_{f_{\phi(n)}\#\mu}(s)-\mathcal{F}_{f\#\mu}(\sbf)|\\
&< 2 \epsilon_{1}
\end{split}
\end{equation*}
As this result holds for any $\epsilon_{1} >0$ we have $\mathcal{F}_{f\#\mu}(\sbf)=\mathcal{F}_{\nu}(\sbf)$ and by injectivity of the Fourrier transform $f\#\mu=\nu$ such that $f$ is measure preserving. 

In the same way for any $\xbf,\xbf' \in \text{supp}(\mu), \epsilon_{2}>0$ and $n$ large enough:  
\begin{equation*}
\begin{split}
\big| \|f(\xbf)-f(\xbf')\|_1-\|\xbf-\xbf'\|_1 \big| &\leq \big| \|f_{\phi(n)}(\xbf)-f_{\phi(n)}(\xbf')\|_1-\|f(\xbf)-f(\xbf')\|_1 \big| \\
&+ \big|\|f_{\phi(n)}(\xbf)-f_{\phi(n)}(\xbf')\|_1-\|\xbf-\xbf'\|_1 \big|< 2 \epsilon_{2}
\end{split}
\end{equation*}
using $f_{\phi(n)} \underset{n \to \infty}{\underset{\rightarrow}{u}} f$ and \eqref{close}. As this result holds true for any $\epsilon_{2} >0$ we have $\|f(\xbf)-f(\xbf')\|=\|\xbf-\xbf'\|$ for any $\xbf,\xbf' \in \text{supp}(\mu)$ which concludes.

\end{proof}

\begin{corr}
Let $\mu,\nu \in \Pm(\R^{p})\times \Pm(\R^{p})$ with compact support. If $\sgw(\mu,\nu)=0$ then $\mu$ and $\nu$ are isomorphic for the distance induced by the $\ell_{1}$ norm on $\R^{p}$, \ie\ $d(\xbf,\xbf')=\sum_{i=1}^{p} |x_{i}-x_{i}'|$ for $(\xbf,\xbf') \in \R^{p} \times \R^{p}$. In particular this implies:
\begin{equation}
\sgw(\mu,\nu)=0 \implies \gw_{2}(d,\mu,\nu)=0
\end{equation} 

\end{corr}
\begin{proof}
If $\sgw(\mu,\nu)=0$ then using the Gromov-Wasserstein properties it implies that for almost all $\thetab \in \Sp^{p-1}$ the projected measures are isomorphic. Moreover since $\mu,\nu$ have compact support, it is bounded and we can directly apply Theorem \ref{cramer_gene} to state the existence of a measure preserving application $f$ as defined in Theorem \ref{cramer_gene}. We consider the coupling $\pi=(Id\times f)\#\mu \in \couplingset(\mu,\nu)$ since $f\#\mu=\nu$. Then we have:
\begin{equation*}
\begin{split}
\int \int |d(\xbf,\xbf')-d(\ybf,\ybf')|^{2} \dr \pi(\xbf,\ybf) \dr \pi(\xbf',\ybf') &= \int \int |d(\xbf,\xbf')-d(f(\xbf),f(\xbf'))|^{2} \dr \mu(\xbf)\dr\mu(\xbf') \\
&=\int \int |\|\xbf-\xbf'\|_1-\|f(\xbf)-f(\xbf')\|_1|^{2} \dr \mu(\xbf)\dr\mu(\xbf') =0
\end{split}
\end{equation*}
Since $f$ is an isometry. This directly implies that $\gw_{2}(d,\mu,\nu)=0$.
\end{proof}

\section{Algorithm for $\sgw$}
\begin{algorithm}[ht]{Sliced Gromov-Wasserstein for discrete measures}
\label{alg:sgw}
\begin{algorithmic}[1]
    \State $p<q$, $\mu= \frac{1}{n} \sum_{i=1}^{n} \delta_{x_{i}} \in \Pm(\R^{p})$ and $\nu=  \frac{1}{n} \sum_{i=1}^{n} \delta_{y_{j}} \in \Pm(\R^{q})$
    \State $ \forall i, x_{i}\leftarrow \D(x_{i}) $, sample uniformly $(\theta_{l})_{l =1,\dots,L} \in \Sp^{q-1}$ 
    \For {$l=1,\dots,L$}
    \State Sort $(\langle x_{i}, \theta_{l}\rangle)_{i}$ and $(\langle y_{j}, \theta_{l}\rangle)_{j}$ in increasing order
    \State Solve \eqref{eq:qap} for reals $(\langle x_{i}, \theta_{l}\rangle)_{i}$ and $(\langle y_{j}, \theta_{l}\rangle)_{j}$, $\sigma_{\theta_{l}}$ is the solution ($\sigma_{\theta_{l}}\in$ Anti-Id or Id )        
    \EndFor
    \State return {\small$\frac{1}{n^{2}L} \sum\limits_{l=1}^{L} \sum\limits_{i,k=1}^{n} \big(\langle x_{i}{-}x_{k},\theta_{l} \rangle^{2}{-}\langle y_{\sigma_{\theta_{l}}(i)}{-}y_{\sigma_{\theta_{l}}(k)},\theta_{l} \rangle^{2}\big)^{2}$}
\end{algorithmic}
\end{algorithm}

In practice, the computation trick presented in Equation~\eqref{eq:linear:computation} can be used to make the complexity of the computation in line 7 linear with $n$.

\section{$SW_{\D}$ and $RISW$}

\begin{figure}[t]
  \centering
  \includegraphics[width=.9\linewidth]{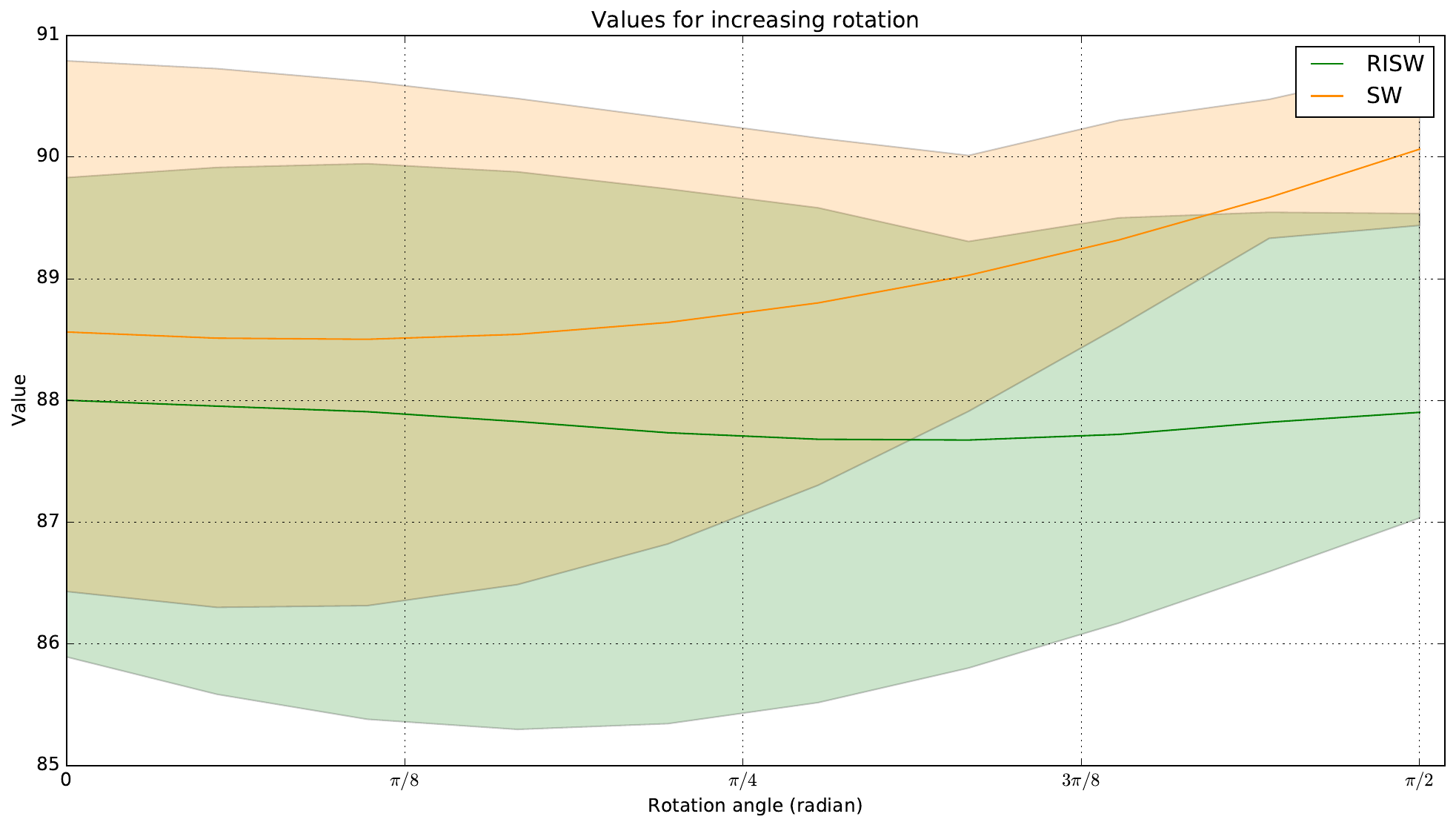}
  \caption{Illustration of $SW$, $RISW$ on spiral datasets for varying rotations on discrete 2D spiral datasets. (left) Examples of spiral distributions for source and target with different rotations. (right) Average value of $SW$ and $RISW$ with $L=20$ as a function of rotation angle of the target. Colored areas correspond to the 20\% and 80\% percentiles. }
  \label{fig:spiral_example_sw}
\end{figure}

Analogously to $SGW$ we can define for the Sliced-Wasserstein distance $SW_{\D}(\mu,\nu)$ for $\mu,\nu \in \Pm(\R^{p})\times \Pm(\R^{q})$ with $p\neq q$ and its rotational invariant counterpart as:

\begin{equation}{}
\label{sgw}
\begin{split}
&SW_{\D}(\mu,\nu)= \fint_{\mathbf{S}^{q-1}}SW(P_{\theta}\#\mu_{\D},P_{\theta}\#\nu) d\theta \\
&RISW(\mu,\nu)= \underset{\D \in \mathbb{V}_{q}(\R^{p})}{\min}SW_{\D}(\mu,\nu)
\end{split}
\end{equation}

where $SW$ is the Sliced-Wasserstein distance. The complexity for computing $SW_{\D}$ is $O(Ln(p+q+\log(n)))$ which is exactly the same complexity as $SGW_{\D}$. With these formulations, we can perform the same experiment as for RISGW on the spiral dataset. The optimisation on the Stiefel manifold is performed using Pymanopt as for $SGW$. Results are reported in Figure \ref{fig:spiral_example_sw}. As one can see, $RISW$ is rotational invariant on average whereas $SW$ is not. One can also note that, due to the sampling process of the spiral dataset, the variance is quite large. This can be explained by the fact that, unlike $SGW$, the Sliced-Wasserstein may realign the distributions without taking the rotation into account. 

\section{Supplementary results for the  $\sgw$ GAN Section}
We give here supplementary results for the $\sgw$ GAN experiment in Fig.~\ref{fig:gan}, where we consider first a generator that outputs 2D
samples, with a two dimensional target, and then a generator that generates 3D samples form a 2D target distribution. Here again, the results are reported
for $1000$ epochs. 
   \begin{figure}
   \centering
   \resizebox{\textwidth}{!}{
\begin{tabular}{ccc|c}
      \includegraphics[width=0.24\textwidth]{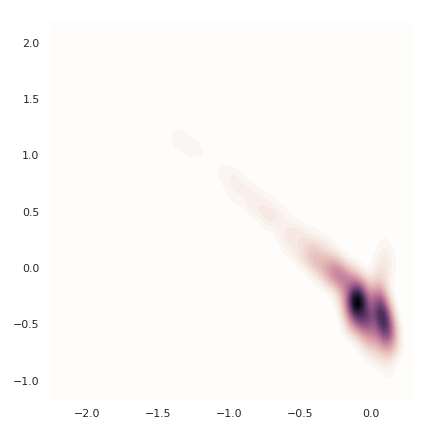}&
      \includegraphics[width=0.24\textwidth]{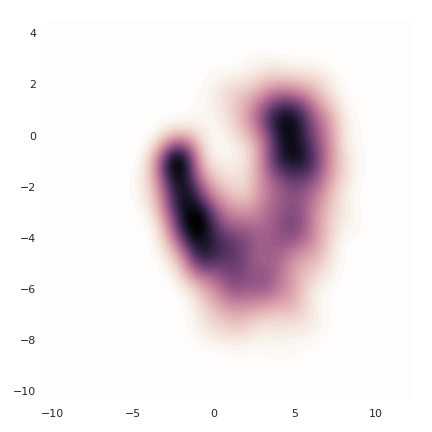}&
      \includegraphics[width=0.24\textwidth]{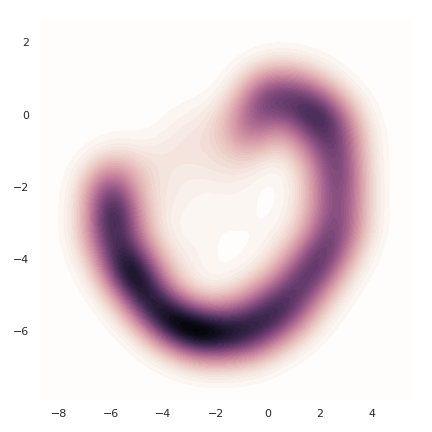}&
      \includegraphics[width=0.24\textwidth]{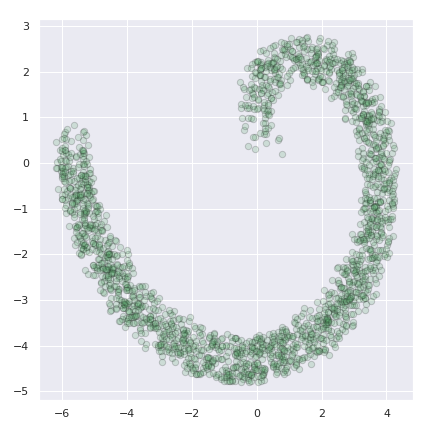}\\
      \includegraphics[width=0.24\textwidth]{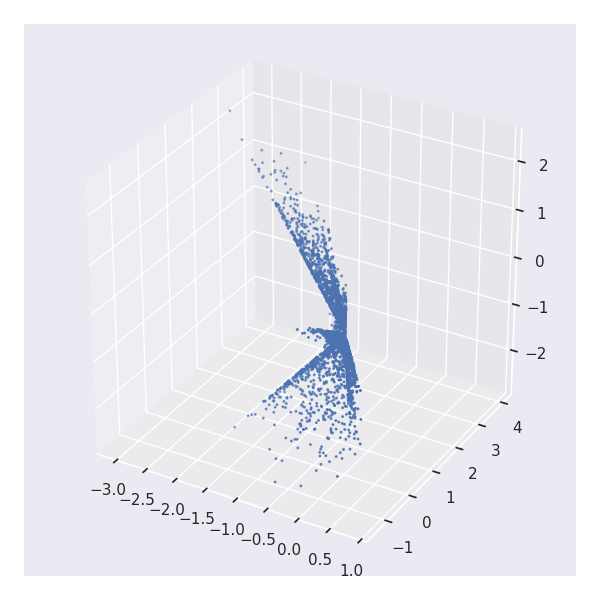}&
      \includegraphics[width=0.24\textwidth]{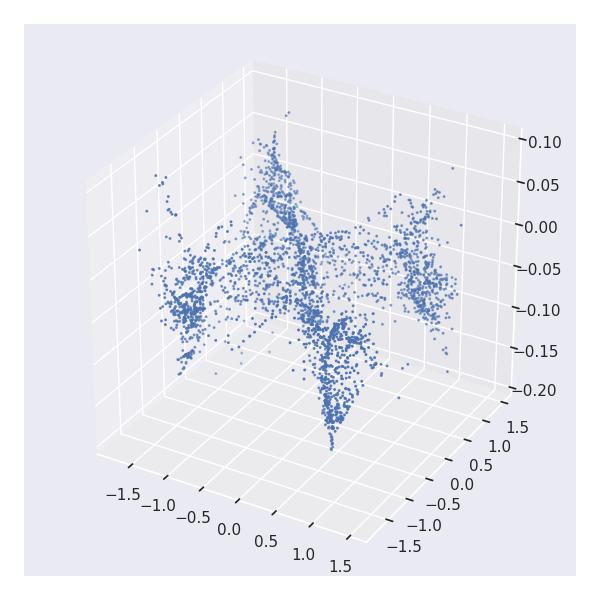}&
      \includegraphics[width=0.24\textwidth]{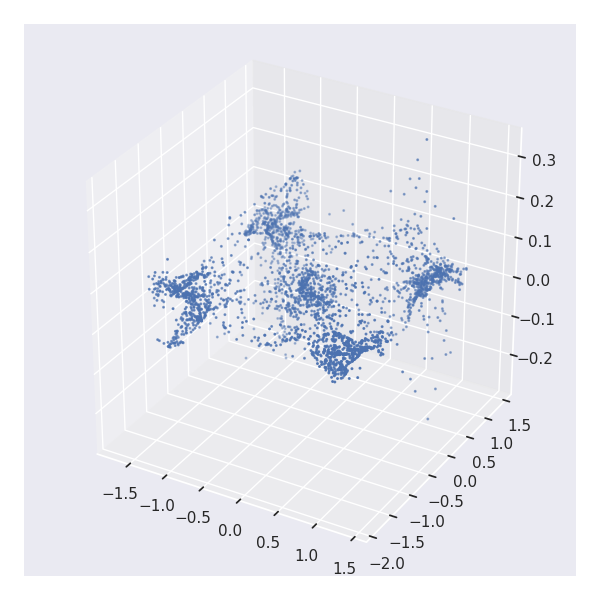}&
      \includegraphics[width=0.24\textwidth]{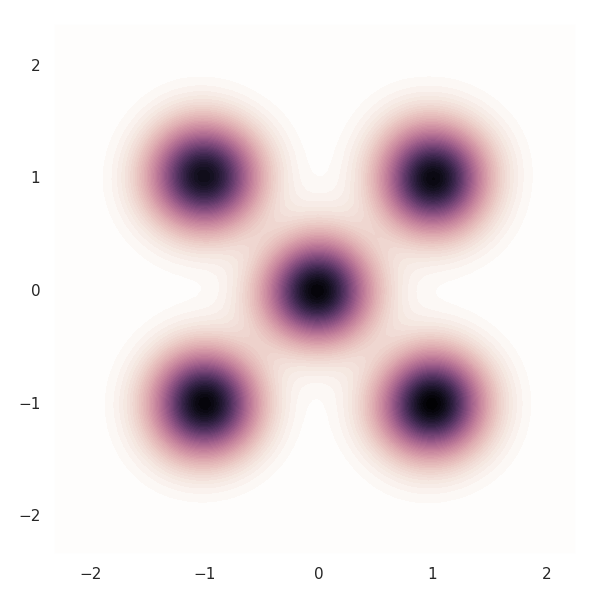}\\
         \end{tabular}}
                         
     \caption{Using $\sgw$ in a GAN loss. The three rows depicts three different examples. First row is 2D (Generator) to 2D (Target) , Second 3D to 2D. First column is initialization, second one is at $100$ Epochs, third one at $1000$. Last column depicts the target distribution.}
     \label{fig:gan}
   \end{figure}

\end{document}